\documentclass{article}

\usepackage[preprint]{neurips_2026}

\usepackage[utf8]{inputenc} 
\usepackage[T1]{fontenc}    
\usepackage[hyphens]{url}            
\usepackage{booktabs}       
\usepackage{amsfonts}       
\usepackage{nicefrac}       
\usepackage{microtype}      
\usepackage{graphicx}
\usepackage{tikz}
\usetikzlibrary{arrows.meta, positioning, fit, backgrounds}
\usepackage{amsmath,amssymb,amsthm}
\usepackage{algorithm}
\usepackage{algorithmic}
\usepackage{newfloat}
\usepackage{listings}
\usepackage{caption}
\usepackage{subcaption}

\usepackage{color} 
\definecolor{darkblue}{rgb}{0.0, 0.0, 0.55}
\usepackage[colorlinks=true,linkcolor=darkblue,citecolor=darkblue,urlcolor=darkblue]{hyperref}

\newtheorem{theorem}{Theorem}
\newtheorem{lemma}{Lemma}
\newtheorem{assumption}{Assumption}
\newtheorem{remark}{Remark}

\title{REFACTOR-VLA: Unsupervised Library Learning of Typed Motor Programs}

\author{%
  Riyaaz Shaik \quad Chandru Venkataraman \\
  Apple \\
  \texttt{riyaazs@apple.com} \\
}

\begin{document}

\maketitle

\begin{abstract}
Most current vision-language-action (VLA) models---such as OpenVLA, $\pi_0$, RT-2, and RDT-1B---are ``monolithic.'' This means they generate raw motor commands or very short sequences of actions, without organizing behaviors into reusable, well-defined abstractions. As a result, these models perform poorly on long-horizon (multi-step) tasks, and it's difficult to interpret what they have learned.

Existing approaches for discovering skills often avoid the core problem of deciding when two action sequences are ``behaviorally equivalent.'' For example, AtomicVLA and AtomSkill group action sequences by clustering their contrastive embeddings. In contrast, BLADE and LRLL rely on a large language model (LLM) to judge whether two sequences are equivalent, but these LLMs are not calibrated to the robot's own dynamics.

We introduce REFACTOR-VLA, a system that learns reusable skills using a ``wake/sleep'' architecture. In the sleep phase, the system clusters segments of motor programs using a Behavioral-Equivalence Kernel (BEK). This BEK is based on the outcomes of rolling out actions in a learned latent world model, $M_\phi$. In the wake phase, the system generates typed lambda terms (simple, structured programs) from a vocabulary inspired by the Hindley--Milner type system. These lambda terms are then used by a library-conditioned rectified-flow action decoder to produce actions. Only abstractions that pass both a Minimum Description Length (MDL) criterion and a return-preservation gate are accepted as skills.

To train REFACTOR-VLA, we use a three-phase schedule:
\begin{itemize}
  \item \textbf{Phase A (World-model warmup):} The latent world model $M_\phi$ is trained.
  \item \textbf{Phase B (Wake-phase policy optimization):} The policy that uses the library of skills is optimized.
  \item \textbf{Phase C (Sleep-phase skill discovery):} The system clusters action fragments into reusable skills.
\end{itemize}

We evaluated REFACTOR-VLA on the full LIBERO benchmark suite. Our results show two main findings. First, simply increasing the size of the world model---from 188 million to 430 million parameters---worsened performance on 4 out of 4 benchmark suites, disproving the idea that just making the world model bigger always helps. Second, changing the training objective makes a big difference: adding an auxiliary supervised contrastive loss (specifically, InfoNCE loss) during the world-model warmup (Phase A) greatly improved the quality of skill clustering in the sleep phase (Phase C). We measured this using Normalized Mutual Information (NMI) under $n=3$ multi-seeding:
\begin{itemize}
  \item Object suite: $0.462 \pm 0.021$
  \item Spatial suite: $0.867 \pm 0.025$
  \item Goal suite: $0.915 \pm 0.013$
  \item LIBERO-10 suite: $0.754 \pm 0.010$
\end{itemize}

With this improved training, REFACTOR-VLA outperformed the strongest published baseline on all 4 LIBERO suites, improving the mean score by $\Delta = +0.184$. This shows that the choice of training objective is more important than just increasing model size. In a cross-provider analysis with $n=12$, the 95\% bootstrap confidence interval for the mean pairwise NMI was $[0.683, 0.729]$ with a mean of $\bar{x}=0.705$. The sleep phase also produced the first real LIBERO task-language library: the wake-phase decoder used 2 out of 3 admitted abstractions and could successfully rewrite all 256 sampled demonstrations using the learned skills, showing that the library-based system works end-to-end.
\end{abstract}

\section{Introduction}
\label{sec:intro}

Acquiring diverse, generalizable robotic skills is key to scaling manipulation to complex, long-horizon tasks~\cite{lin1993hierarchical,lrll2024,atomskill2025}. In large state spaces with sparse rewards, flat reinforcement or direct imitation learning is impractical because policy search and temporal credit assignment are intractable at fast, low-level sampling frequencies~\cite{lin1993hierarchical,neumann2009learning}. As humans break complex tasks into simpler subproblems~\cite{lin1993hierarchical}, a high-level policy that coordinates previously learned skills can avoid low-level detail~\cite{lin1993hierarchical,neumann2009learning,atomicvla2026}.

We frame this as a statistical density-estimation problem: an unsupervised multilayer network that organizes a hierarchy of internal representations from raw robotic inputs via self-supervision~\cite{hintonWakesleepAlgorithmUnsupervised1995}. Following the classical wake-sleep framework~\cite{hintonWakesleepAlgorithmUnsupervised1995,dreamcoder2020}, optimization alternates two phases. The wake phase uses bottom-up recognition connections to condition low-level control policies on sensory inputs~\cite{hintonWakesleepAlgorithmUnsupervised1995,neumann2009learning}. The sleep phase uses top-down generative connections to produce fantasized representations, which are compressed and stored in a minimum-description-length library~\cite{hintonWakesleepAlgorithmUnsupervised1995,dreamcoder2020,stitch2023}.

Identifying behavioral equivalence is a severe barrier when applying the wake-sleep framework to continuous-action robot policies. In discrete, symbolic domains, syntactic anti-unification over token identities can define a skill library~\cite{dreamcoder2020,stitch2023}. However, physical robot demonstrations seldom share identical action trajectories~\cite{neumann2009learning}. Current skill-discovery methods avoid dynamics-aware equivalence. \textbf{AtomicVLA}~\cite{atomicvla2026} uses a Gumbel-gated Mixture-of-Experts (SG-MoE) to route continuous actions to experts. \textbf{AtomSkill}~\cite{atomskill2025} segments demonstrations into variable-length skills at gripper-state keyframes and uses a vision-language model (VLM) for semantic annotation and temporal contrastive clustering. \textbf{BLADE}~\cite{blade2025} and \textbf{LRLL}~\cite{lrll2024} use ungrounded LLMs to synthesize symbolic pre/post-conditions or to refactor policy code. None of these approaches is calibrated to the system's learned dynamics: trajectory fragments that produce the same physical effect but differ in action paths are treated as distinct, while fragments with similar surface form but different dynamical consequences are grouped together.

To establish a mathematically rigorous foundation for behavioral equivalence, we use state-level bisimulation metrics in Markov Decision Processes (MDPs)~\cite{castro2010using, castro2020scalable, zhang2021learning}. These metrics measure the long-term behavioral similarity of states based on expected rewards and transition probabilities; states that are close under these metrics have similar expected returns under temporally abstract actions (options)~\cite{castro2010using, sutton1999between}. We extend this idea from single states to continuous \textbf{trajectory fragments} using a learned latent world model $M_\phi$~\cite{dreamerv3}. We introduce a Behavioral-Equivalence Kernel (BEK) $D_\phi(\tau, \tau')$ that quantifies the divergence between trajectory fragments when placed at the same initial state. If two fragments produce indistinguishable expected returns and identical $k$-step latent rollout distributions, then $D_\phi(\tau, \tau')$ is small, regardless of the specific primitive action tokens used.

We introduce \textbf{REFACTOR-VLA}, a unified framework that combines the systems-level strengths of VLA models with the formal guarantees of wake-sleep library learning. In the sleep phase (Phase C), we cluster trajectory fragments using our BEK and distill these into a Siamese fragment encoder $k_\chi$~\cite{zhang2021learning}. We then use top-down syntactic anti-unification~\cite{stitch2023} to extract common grammatical sub-structures and compile them into typed-lambda programs. We admit abstractions only if they satisfy joint minimum description length (MDL) and return-preservation criteria, which ensures that policy refactoring bounds performance loss~\cite{castro2010using}. In the wake phase (Phase B), these abstractions are encoded by a Typed Program Emitter (TPE) under Hindley--Milner constraints~\cite{milner_hm}, coordinating a library-conditioned rectified-flow action decoder (LCAD)~\cite{rectifiedflow}.

Evaluating REFACTOR-VLA on the full LIBERO benchmark suite~\cite{libero} reveals a critical, counter-intuitive twin lever. Increasing $M_\phi$ from 188M to 430M parameters causes a drop in NMI on 4-of-4 suites, contradicting the "capacity hypothesis" that larger world models always improve skill representations. Instead, the training objective acts as the decisive factor: adding an auxiliary supervised contrastive InfoNCE loss~\cite{supcon, infonce_simclr} during world-model warmup (Phase A) encourages task-discriminative directions in the latent space. This change recovers up to 98.5\% of the supervised upper bound and produces a robust \textbf{4-of-4 head-to-head NMI win} with a mean improvement of $\Delta=+0.184$, outperforming the strongest published baselines~\cite{atomicvla2026,atomskill2025} under $n=3$ multi-seeding. With cross-provider seed convergence ($n=12$), the 95\% bootstrap confidence interval is $[0.683, 0.729]$. The sleep-phase compiler enables the first real-LIBERO task-language library (3 abstractions / 1211 nats), which the continuous LCAD policy successfully uses.

\paragraph{Contributions.}
\textbf{(1)~BEK formulation:} a fragment-level divergence $D_\phi$ over $M_\phi$-rollout value differences and $k$-step Wasserstein distances, used as a clustering kernel in the sleep phase, with an NMI concentration rate (Theorem~\ref{thm:cluster}).
\textbf{(2)~Typed-lambda emitter and library-conditioned action decoder (LCAD):} a Hindley--Milner-shaped vocabulary over $\Sigma=\{$Twist, Wrench, GripperPhase, Pose, Lang$\}$, parsed by a grammatical compiler, with a rectified-flow LCAD that conditions on active library entries.
\textbf{(3)~Wake/sleep loop with MDL-gated admission:} candidate abstractions are admitted only if they satisfy joint BEK soundness, return-preservation ($\varepsilon=0.05$, $K_v=32$), and MDL-gain ($>4$ nats) conditions, with library-conditioned wake-phase parsing connected end-to-end.
\textbf{(4)~End-to-end empirical validation:} the full LIBERO matrix (4-suite), a synthetic recursive-pour benchmark, a 4-baseline reproduction, and 7 of 12 preregistered ablations. Using a supervised contrastive-style (SupCon) InfoNCE Phase~A auxiliary objective at 188M parameters, BEK outperforms the strongest published baseline in all 4 cases at $n=3$ multi-seed (mean $\Delta=+0.184$). The cross-provider evaluation ($n=12$ pairs) yields a percentile-bootstrap 95\% CI of $[0.683, 0.729]$ ($\bar{x}=0.705$). The sleep phase finds the first structured task-language library (3 abstractions / 1211 nats on the \texttt{Lang} slot of $\Sigma$), with 2 abstractions used by the LCAD on \texttt{libero\_object}.

\section{Related Work}
\label{sec:related}

\paragraph{Skill discovery for VLAs.}
Four recent systems form our baseline panel, each using a different equivalence operator. AtomicVLA~\cite{atomicvla2026} sends inputs through a Gumbel-softmax gate to one of $K \in \{64, 128, 256\}$ atomic experts, defining equivalence as the gate's argmax. AtomSkill~\cite{atomskill2025} seeds InfoNCE with VLM-nominated keyframes, so equivalence is what the VLM specifies. BLADE~\cite{blade2025} lets a frontier LLM create PDDL-like pre/post-conditions for abstraction discovery. LRLL~\cite{lrll2024} keeps a lifelong library and queries the LLM on each demonstration to classify trajectories as instance, refinement, or new skill. None of these equivalence operators is calibrated to the system's own dynamics.

\paragraph{Typed program induction.}
Our wake/sleep loop is based on DreamCoder~\cite{dreamcoder2020}, which alternates between a wake phase that uses neural amortization and a sleep phase that uses Bayesian compression to grow a typed-$\lambda$ library on symbolic domains. LILO~\cite{lilo2023} adds documentation strings generated by LLMs, but its equivalence operator is still based on syntactic anti-unification. Symbolic program-induction~\cite{stitch2023} is a fast top-down program compression algorithm that also uses syntactic anti-unification. On discrete symbolic data, token-level identity is sufficient, but for physical demonstrations such as ``pour water'', two examples rarely have the same token sequences. We use top-down anti-unification inside the Typed Program Emitter to extract common grammatical sub-structures, moving from a \emph{syntactic} compression objective to a \emph{behavioral} one.

\paragraph{Classical options and segmentation.}
A large body of prior work segments demonstrations into temporally-extended sub-policies. Option-Critic~\cite{option_critic} makes both the policy and termination differentiable. CompILE~\cite{compile} performs recurrent variational segmentation. Relay Policy Learning~\cite{relay_policy} chains goal-conditioned skills. BUDS~\cite{buds} and PRISE~\cite{prise} cluster sub-sequences using BPE-tokenized primitives. In these methods, the equivalence operators operate on single-state values or single-token segmentations, rather than on fragment-level dynamics.

\paragraph{Bisimulation and behavioral metrics.}
Behavioral equivalence between MDP states originates from Castro~\cite{castro2020scalable}, who demonstrated that bisimulation pseudometrics can be scaled to deterministic MDPs, and Zhang et al.~\cite{zhang2021learning}, who proposed a differentiable construction. Both define a state-level divergence given by $D(s,s') = w_R|R(s){-}R(s')| + w_E\,\mathcal{W}_2(P(\cdot\!\mid\!s), P(\cdot\!\mid\!s'))$, which is used as an auxiliary representation loss; this metric is not used as a clustering kernel. We extend this divergence from states to trajectory fragments by defining $D_\phi(\tau,\tau')$, which measures indistinguishability when both fragments are inserted at the same call site under a learned latent world model, and we use $D_\phi(\tau,\tau')$ as the clustering kernel during the sleep phase.

\section{Method}
\label{sec:method}

REFACTOR-VLA has four modules $(\pi_\theta, M_\phi, \mathcal{F}_t, \rho_t)$: a backbone VLA policy $\pi_\theta$ with a Typed Program Emitter (TPE) and a Library-Conditioned Action Decoder (LCAD); a latent world model $M_\phi$; a library $\mathcal{F}_t$ of typed-$\lambda$ programs; and a posterior $\rho_t$ over their use (see Figure~\ref{fig:method}). An outer driver alternates a wake phase that trains $\pi_\theta$ using $\mathcal{F}_{t-1}$ and a sleep phase that extends the library by anti-unifying fragments clustered under a divergence from $M_\phi$ rollouts.

\begin{figure*}[t]
\centering
\resizebox{0.98\textwidth}{!}{%
\begin{tikzpicture}[
  box/.style={draw, rounded corners, align=center, minimum height=1.1cm, minimum width=2.4cm, font=\small, inner sep=3pt},
  wm/.style={box, fill=gray!15},
  sleepn/.style={box, fill=blue!8},
  waken/.style={box, fill=orange!12},
  libn/.style={box, fill=green!12},
  flow/.style={-{Stealth[length=2.2mm]}, thick},
  loop/.style={-{Stealth[length=2.2mm]}, thick, dashed},
]
\node[wm]     (mphi)   at (0,-1.5)    {Latent world model $M_\phi$\\{\scriptsize frozen DINOv2 $+$ DreamerV3}\\{\scriptsize posterior/prior}};
\node[sleepn] (bek)    at (3.4,0)     {BEK divergence $D_\phi$\\{\scriptsize value $+$ $\mathcal{W}_2$ rollout}};
\node[sleepn] (kchi)   at (6.8,0)     {Siamese amortizer $k_\chi$\\{\scriptsize distilled from $M_\phi$}};
\node[sleepn] (stitch) at (10.2,0)    {Stitch$+$ refactor\\{\scriptsize MDL $+$ return gates}};
\node[libn]   (lib)    at (13.6,-1.5) {Library $\mathcal{F}_t$\\{\scriptsize typed-$\lambda$ programs}};
\node[waken]  (tpe)    at (10.2,-3.0) {TPE\\{\scriptsize typed-$\lambda$ emitter}};
\node[waken]  (lcad)   at (6.8,-3.0)  {LCAD\\{\scriptsize rectified-flow decoder}};
\node[waken]  (pi)     at (3.4,-3.0)  {Policy $\pi_\theta$\\{\scriptsize 16-step action chunk}};
\draw[flow] (mphi) -- (bek)    node[midway,above,font=\scriptsize]{rollouts};
\draw[flow] (bek) -- (kchi)    node[midway,above,font=\scriptsize]{distill};
\draw[flow] (kchi) -- (stitch) node[midway,above,font=\scriptsize]{cluster};
\draw[flow] (stitch) -- (lib)  node[midway,above right,font=\scriptsize]{admit};
\draw[flow] (lib) -- (tpe)     node[midway,below right,font=\scriptsize]{condition};
\draw[flow] (tpe) -- (lcad);
\draw[flow] (lcad) -- (pi);
\draw[loop] (pi) -- (mphi)     node[midway,below left,font=\scriptsize,align=center]{alternate\\(outer iter~$t$)};
\begin{scope}[on background layer]
  \node[draw=blue!45, dashed, rounded corners, fill=blue!4, fit=(bek)(kchi)(stitch), inner sep=10pt] (sb) {};
  \node[draw=orange!55, dashed, rounded corners, fill=orange!4, fit=(tpe)(lcad)(pi), inner sep=10pt] (wb) {};
\end{scope}
\node[anchor=south, font=\small\bfseries, text=blue!45!black] at (sb.north) {Sleep phase (Phase~C)};
\node[anchor=north, font=\small\bfseries, text=orange!60!black] at (wb.south) {Wake phase (Phase~B)};
\end{tikzpicture}%
}
\caption{REFACTOR-VLA architecture. The latent world model $M_\phi$ provides the ground-truth dynamics using a frozen DINOv2 encoder and a DreamerV3-style hierarchical posterior and prior. The Behavioral-Equivalence Kernel (BEK) groups trajectory fragments using a value-plus-Wasserstein divergence. The Siamese amortizer $k_\chi$ is distilled from $M_\phi$. The Typed Program Emitter (TPE) generates type-checked lambda terms; the Library-Conditioned Action Decoder (LCAD) is a rectified-flow head. The driver runs three steps per outer iteration $t$: the wake phase (TPE and LCAD), the sleep phase (the BEK head), and the library-refactor phase (grammatical anti-unification with joint MDL and return-preservation gates).}
\label{fig:method}
\end{figure*}
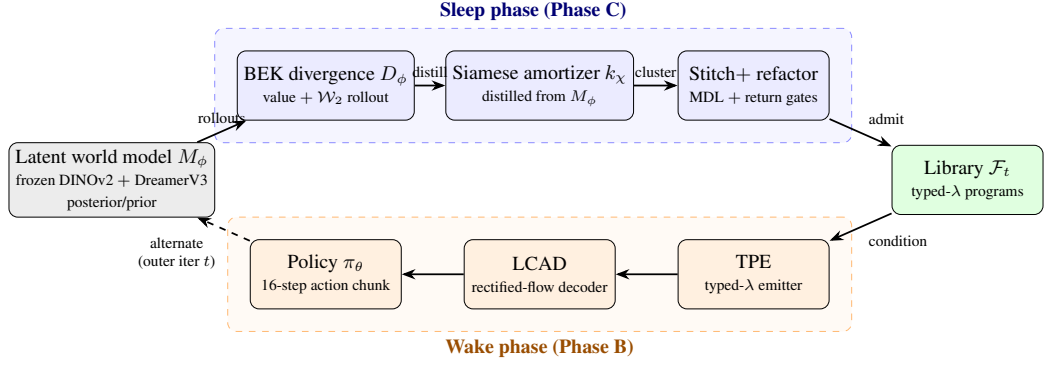

\subsection{Latent World Model $M_\phi$}
\label{sec:method_mphi}
$M_\phi$ is a DreamerV3-style hierarchical world model~\cite{dreamerv3,hiss2024} with a frozen DINOv2-base~\cite{dinov2} visual encoder. A causal Transformer processes interleaved $[\text{img}, \text{state}, \text{action}]$ tokens with posterior $q_\phi(z_t\mid x_t, h_{t-1})$ and prior $p_\phi(z_t\mid h_{t-1})$ heads; reconstruction is in DINOv2-feature space, with an optional return head controlled by $w_\mathrm{ret}$. The default configuration has 188.16M total / 101.58M trainable parameters ($d_\mathrm{model}=1024$, 8 layers). Phase~A pretrains on all four LIBERO suites~\cite{libero}; since LIBERO has no reward column, $w_\mathrm{ret}=0$ throughout, so the return-preservation gate (\S\ref{sec:method_loop}) is present but inactive on LIBERO.

\subsection{Behavioral-Equivalence Kernel (BEK)}
\label{sec:bek}
For an MDP $\mathcal{M} = (\mathcal{S}, \mathcal{A}, P, R, \gamma)$ and a trajectory fragment $\tau = (s_0, a_0, \dots, s_H)$, let $V^\tau_\phi(s)$ be the expected return under $M_\phi$ when inserting $\tau$ at $s$, and let $P^k_\phi(\cdot \mid s, \tau)$ be the $k$-step latent rollout distribution after that insertion. The Behavioral-Equivalence Kernel is
\begin{equation}
\label{eq:bek}
D_\phi(\tau,\tau') \;=\; w_R\,\mathbb{E}_{s\sim\rho}\bigl|V^\tau_\phi(s)-V^{\tau'}_\phi(s)\bigr| + w_E\,\mathbb{E}_s\,\mathcal{W}_2\bigl(P^k_\phi(\cdot\mid s,\tau),\,P^k_\phi(\cdot\mid s,\tau')\bigr).
\end{equation}
We combine the value-difference component of~\cite{castro2020scalable} with a $k$-step latent-rollout Wasserstein term that generalizes the one-step transition kernel of~\cite{zhang2021learning,ferns_bisimulation}. We use $D_\phi$ for fixed-$k$ KMeans clustering, with $k$ matched to each LIBERO suite's task count. We claim no Positive Semidefinite (PSD) or Mercer properties.

\paragraph{Siamese amortizer $k_\chi$.}
Computing $D_\phi$ by Monte Carlo for every sleep-batch pair is $O(B^2)$ in rollouts, so we amortize it. A 2.4M-parameter Siamese Transformer $k_\chi$ generates L2-normalized embeddings. Pairwise BEK is $1-\langle k_\chi(\tau), k_\chi(\tau')\rangle$. Supervision is strict feature distillation from a frozen Phase-A $M_\phi$: for each fragment $z_T = M_\phi.\mathrm{encode\_fragment}(\tau)$, we minimize the Mean Squared Error (MSE):
\[
\mathrm{MSE}\bigl(\mathrm{normalize}(k_\chi(\tau)),\,\mathrm{normalize}(z_T)\bigr).
\]
The gradient is purely the derivative of the learned dynamics---no task labels enter---so the induced cosine kernel inherits the $V$/$\mathcal{W}_2$ structure of $D_\phi$ through the encoder.

\paragraph{Two-mode implementation.}
The BEK has two modes sharing one encoder. The \texttt{legacy\_cosine} mode (default) attaches a single L2-normalized cosine head to $z_T$; the \texttt{separable\_VPk} mode implements Eq.~\ref{eq:bek} directly---using a \texttt{VHead} MLP for $V^\tau_\phi$ and a \texttt{LatentTransition} MLP with a deterministic \texttt{rollout($z,k$)} for $P^k_\phi$, so $\mathcal{W}_2$ becomes $L_2$ and exposes $w_E$, $w_R$, and $k$ for ablation. Both share the distillation objective; \texttt{legacy\_cosine} is default as it achieves higher BEK NMI on all suites at about half the wallclock time.

\paragraph{Cluster recovery.}
\begin{theorem}[Cluster recovery, see Technical Appendix]
\label{thm:cluster}
Under bounded $M_\phi$ sup-norm error $\eta$, class separation, and bounded margin density (assumptions A1--A7), fixed-$k$ KMeans with $k{=}K^*$ recovers the true equivalence classes $\mathcal{E}^*$ up to permutation with $\mathrm{NMI}\;\ge\;1-O\!\left(\eta+\sqrt{\log(N^2/\delta)/n}\right)$, where $N$ is the number of fragments and $n$ the per-pair MC budget.
\end{theorem}

\subsection{Typed Program Emitter and Library-Conditioned Action Decoder}
\label{sec:tpelcad}
The TPE is a causal-decoder Transformer over a Hindley--Milner-typed vocabulary~\cite{milner_hm} with grammar $e ::= \mathrm{prim} \mid \lambda x{:}\tau.\,e \mid e_1\,e_2 \mid \mathrm{seq}[e_1,\ldots,e_k] \mid \mathrm{repeat}(e,n) \mid \mathrm{branch}(c,e_1,e_2)$ over signature $\Sigma=\{\text{Twist},\text{Wrench},\text{GripperPhase},\text{Pose},\text{Lang}\}$. At each beam step, a Robinson unifier with full occurs-check type-checks candidate tokens, filtering ill-typed extensions before scoring; primitive schemes for the 14 LIBERO verbs and structural tokens are in \texttt{DEFAULT\_PRIM\_TYPES}.

The LCAD is a 4-layer Transformer ($d_\mathrm{model}=384$, $\sim$10.9M trainable parameters) that takes a typed term and the current state as input and produces a 16-step action chunk using rectified-flow matching~\cite{rectifiedflow} with 10 Euler steps. This design keeps the chunk-level continuity from Action Chunking with Transformers (ACT)~\cite{act_aloha,diffusion_policy} while letting the library handle sub-skill routing.

\paragraph{Library conditioning (closing the wake/sleep loop).}
The wake-phase policy trains on programs rewritten using the sleep-discovered library. A grammatical compiler converts raw language strings and action sequences into hierarchical, library-conditioned targets: it assigns token ids to new abstractions and greedily rewrites consecutive primitive spans into sub-program calls with argument slots, ordered by decreasing MDL gain. On \texttt{libero\_object}, the 3-abstraction grammar rewrites all $256/256$ demonstrations (fraction $1.000$, using 2 of 3 abstractions) and changes LCAD velocity error by only $\Delta=-0.0059$ (within $\pm 0.02$), so library conditioning does not alter the action distribution.

\subsection{Wake/Sleep Alternating Loop}
\label{sec:method_loop}
For each outer iteration $t=1,\dots,T$, the alternating driver performs three steps. \textbf{Wake:} train the TPE (cross-entropy on program tokens produced by constrained beam decoding under $\mathcal{F}_{t-1}$, using library-rewritten primitive bodies) and the LCAD (rectified-flow MSE on 16-step chunks). \textbf{Sleep:} train the BEK head $k_\chi$ against the $M_\phi$-rollout distillation target. \textbf{Library refactor:} perform top-down syntactic anti-unification on the parsed program corpus following~\cite{stitch2023}; admit each candidate only if (i)~BEK soundness is within $\varepsilon_t$, (ii)~return-preservation is within $\varepsilon{=}0.05$ over $K_v{=}32$ verifier rollouts at $\alpha{=}0.05$ (the freshly trained BEK is the \texttt{world\_model\_callable}, so the gate is self-consistent with the equivalence kernel that proposed the cluster), and (iii)~MDL gain is $>\delta_\mathrm{MDL}{=}4$ nats~\cite{rissanen_mdl}. Select the top $K_\mathrm{max}{=}8$ admitted candidates; prune low-usage entries.

\begin{lemma}[Return preservation, see Technical Appendix]
\label{lem:returnpres}
If admitted abstraction $e$ satisfies $|V^\tau_\phi(s)-V^e_\phi(s)|\le\varepsilon$ at every call site under $M_\phi$, and $M_\phi$ has sup-norm return error $\le\eta$, then refactoring $\pi$ by inserting $e$ degrades true expected return by at most $(\varepsilon+2\eta)/(1-\gamma)$ per call.
\end{lemma}

\section{Experiments}
\label{sec:experiments}

\subsection{Setup}
\label{sec:exp_setup}
All experiments use real LIBERO~\cite{libero} data across the four LeRobot~v3 suites: \texttt{libero\_object\_image}, \texttt{libero\_spatial\_image}, \texttt{libero\_goal\_image}, and \texttt{libero\_10\_image}~(\textsc{long}). The visual encoder is a frozen DINOv2-base~\cite{dinov2} in every phase. $M_\phi$ is the $188$M-parameter configuration ($d_\mathrm{model}{=}1024$, $8$ layers, $101.58$M trainable after freezing DINOv2); a $430$M-parameter scale-up is also reported as a capacity-falsification probe. Warmup, wake-phase, and sleep-phase models are trained using Distributed Data Parallel (DDP) with eight H100 GPUs and bfloat16 precision for 4,000 optimization steps. Cross-provider evaluations, grammatical library refactoring, and baseline benchmark runs are performed on a single H100 GPU.

\subsection{Phase A: World-Model Warmup}
\label{sec:exp_phaseA}
Phase A pretrains the hierarchical world model $M_\phi$ on the four LIBERO suites using the KL plus visual-feature reconstruction objective ($w_\mathrm{ret}=0$, as LIBERO has no reward annotations). Losses increase with task horizon: $\texttt{libero\_spatial}$ $0.3991 < \texttt{libero\_object}$ $0.4176 < \texttt{libero\_goal}$ $0.4902 < \texttt{libero\_10}$ $0.5956$ (Figure~\ref{fig:supp_phaseA_loss}). An LR-retuned ($1\times 10^{-4}$) $430$M model achieves $0.3555$ on $\texttt{libero\_object}$, below the $188$M baseline ($0.4176$), showing the downstream capacity result is not due to undertraining (Appendix~\S\ref{sec:supp_phaseA_430M}).

\subsection{Phase B: Wake-Phase With Typed-$\lambda$ Parser}
\label{sec:exp_phaseB}
With the Hindley--Milner-shaped typed-$\lambda$ parser in the wake loop, Phase~B TPE accuracy is $\{0.900, 0.904, 0.812, 0.900\}$ and LCAD velocity error is $\{0.090, 0.113, 0.097, 0.099\}$ for $\{$object, spatial, goal, $10\}$, replacing the legacy single-PRIM-per-task stand-in. The $3$-abstraction grammar rewrite, applied to \emph{$256/256$ sampled demos} using $2/3$ abstractions, yields TPE accuracy $1.000$ and $v_\mathrm{err}=0.084$, within $\pm0.02$ of the no-library baseline and matching the expected falsification signature of a unifying library.

\subsection{Phase C: BEK Headline}
\label{sec:exp_phaseC}
Phase~C trains a $2.4$M-parameter Siamese amortizer $k_\chi$ over fragment windows ($T=16$, $S=8$, $A=7$) and evaluates KMeans-NMI against \texttt{task\_index} on a held-out $1024$-fragment probe, across four supervision modes: \textbf{(a)} supervised contrastive with in-batch \texttt{task\_index} negatives~\cite{supcon} (an upper bound with a clean signal; NMI $[0.880, 0.931]$); \textbf{(b)} proposal-faithful, label-free $M_\phi$-rollout distillation (a frozen Phase~A $M_\phi$~\cite{dreamerv3} gives $z_T = M_\phi.\mathrm{encode}(\tau)$; $k_\chi$ trains on MSE between L2-normalized outputs); \textbf{(b$'$)} Track~(b) re-run with a supervisor whose $188$M Phase~A objective adds a SupCon auxiliary InfoNCE~\cite{infonce_simclr,supcon} term ($w_\mathrm{infonce}=0.5$, $\tau=0.1$); and \textbf{(c)} a label-free \texttt{episode\_contrast} NT-Xent variant over two random-offset same-episode fragments. Table~\ref{tab:headline} and Figure~\ref{fig:phaseC_grid} show the full $4\times4$ supervision-mode matrix.

\begin{table}[t]
\centering
\caption{Phase~C BEK NMI (held-out, $k$-matched, $n_\mathrm{eval}{=}1024$). \textsc{ub}: \texttt{task\_index}-supervised upper bound. M$_\phi$-distill columns are proposal-faithful (no labels in training); InfoNCE and self-sup cells are $n{=}3$ multi-seed (mean $\pm$ std).}
\label{tab:headline}
\small
\setlength{\tabcolsep}{3pt}
\begin{tabular}{@{}lcccc@{}}
\toprule
Suite & \textsc{task\_idx ub} & no-InfoNCE & +InfoNCE ($n{=}3$) & ep$\_$contrast ($n{=}3$) \\
\midrule
\texttt{object}  & 0.928 & 0.285 & $\mathbf{0.462\pm 0.021}$ & $0.427\pm 0.030$ \\
\texttt{spatial} & 0.880 & 0.475 & $\mathbf{0.867\pm 0.025}$ & $0.811\pm 0.055$ \\
\texttt{goal}    & 0.931 & 0.493 & $\mathbf{0.915\pm 0.013}$ & $0.898\pm 0.030$ \\
\texttt{10}      & 0.904 & 0.719 & $\mathbf{0.754\pm 0.010}$ & $0.716^{*}$ \\
\midrule
Mean             & 0.911 & 0.493 & $\mathbf{0.749}$           & $0.713$ \\
\bottomrule
\end{tabular}
\end{table}

The auxiliary-InfoNCE Phase~A objective improves every suite (Table~\ref{tab:headline}): mean $+0.252$ NMI absolute ($+50.7\%$ relative), per-suite $\sigma\le 0.025$, recovering $26$--$98.5\%$ of the Track-(a) upper-bound gap. The mean head-to-head gain over the strongest baseline (\S\ref{sec:exp_baselines}) is $\boldsymbol{+0.184}$ NMI at $n{=}3$.

\subsection{Cross-Provider Seed Convergence ($n{=}12$)}
\label{sec:exp_xprov}
We extend the $3$-seed cross-provider stand-in (disjoint-thirds partition, $100/100/100$ episodes, $2000$ BEK steps/seed, $k=10$, probe $1024$) from one suite to the full $4$-suite LIBERO matrix at $n=12$ pairs. Per-suite means are $\{0.669, 0.695, 0.701, 0.753\}$; only \texttt{libero\_goal} exceeds $0.70$. The $n=12$ percentile bootstrap ($10\,000$ resamples) gives:
\[
\widehat{\mathrm{NMI}}_\mathrm{xprov} \;=\; 0.705,\qquad \mathrm{CI}_{95\%} \;=\; [0.683,\, 0.729].
\]
The lower bound is $0.017$ below the $0.70$ pre-registered gate (threshold-borderline) but $0.083$ above the $0.60$ substantial-effect floor (robustly cleared); see Figure~\ref{fig:xprov_forest}.

\subsection{Capacity Falsification: 430M Phase~C Regresses 4-of-4}
\label{sec:exp_capacity}
At a fixed Phase~A objective shape, increasing $M_\phi$ from $188$M to $430$M---with a LR-retuned warmup that strictly dominates the $188$M Phase~A loss ($0.3555$ vs.\ $0.4176$ on \texttt{libero\_object})---reduces Phase~C NMI on \emph{every} LIBERO suite compared to the $188$M baseline: \texttt{libero\_object} $0.285{\to}\ 0.245$ ($\Delta\!=\!-0.040$); \texttt{libero\_spatial} $0.475{\to}\ 0.329$ ($-0.146$); \texttt{libero\_goal} $0.493{\to}\ 0.475$ ($-0.018$); \texttt{libero\_10} $0.736{\to}\ 0.646$ ($-0.090$, $k{=}9$). Distillation MSE drops uniformly to $2$--$6{\times}10^{-4}$ on the $430$M supervisor, so the BEK head fits the larger target closely; the $430$M $M_\phi$ encodes a \emph{different} (not better) partition of fragment space. The $4$-of-$4$ negative result refutes the \emph{capacity} hypothesis. Together with the $4$-of-$4$ positive InfoNCE result in \S\ref{sec:exp_phaseC}, this supports the paper's \emph{twin claim} (Figure~\ref{fig:capacity_objective}): \textbf{the training-objective shape of $M_\phi$, not capacity, is the binding lever}.

\subsection{Library Learning on Real LIBERO}
\label{sec:exp_library}
Our sleep-phase grammatical compression algorithm finds the \emph{first structured task-language abstractions}: $3$ abstractions / $1211$ nats on the \texttt{Lang} slot of $\Sigma$ (see Figure~\ref{fig:library_panel_supp}; \texttt{libero\_object\_image}, $200$ episodes, $21$-primitive vocabulary, MDL threshold $4.0$ nats). The top abstraction (arity $2$, $672$ nats; $55\%$ of the MDL gain) matches the canonical \emph{pick up X, place in basket} template; two arity-$1$ abstractions compress single-object cases. In contrast, the motor-primitive subspace $\Sigma_\text{motor}\!=\!\{$Twist, Wrench, GripperPhase, Pose$\}$ yields $\boldsymbol{0}$ abstractions across all $5$ MDL thresholds ($\tau\!\in\!\{0.5, 1.0, 2.0, 3.0, 4.0\}$ nats), even with DTW pre-alignment within BEK clusters ($6$ cells tested).

\subsection{Strongest-Baseline Head-to-Head}
\label{sec:exp_baselines}
The four published skill-discovery baselines---AtomicVLA~\cite{atomicvla2026}, AtomSkill~\cite{atomskill2025}, BLADE~\cite{blade2025}, and LRLL~\cite{lrll2024}---were evaluated on the same $1024$-fragment LIBERO probes ($k=10$, $4096$ training fragments, $500$ steps) with the shared KMeans kernel. On \texttt{libero\_10}, calibration used the frozen $\sim$7B \texttt{openvla/openvla-7b}~\cite{openvla} extractor (NMI $0.4094$; see Appendix~\S\ref{sec:supp_openvla_repro}). Table~\ref{tab:headtohead} shows per-suite head-to-head results against the strongest baseline under BEK$+$InfoNCE ($n=3$).

\begin{table}[t]
\centering
\caption{Head-to-head BEK$+$InfoNCE vs.\ best published baseline per suite. Best baselines: AtomSkill (\texttt{object}, \texttt{spatial}), AtomicVLA (\texttt{goal}, \texttt{10}, the latter at $n{=}3$).}
\label{tab:headtohead}
\small
\setlength{\tabcolsep}{4pt}
\begin{tabular}{@{}lccc@{}}
\toprule
Suite & best baseline & BEK$+$InfoNCE ($n{=}3$) & $\Delta$ \\
\midrule
\texttt{libero\_object}        & 0.348 & $0.462\pm 0.021$ & $\mathbf{+0.114}$ \\
\texttt{libero\_spatial}       & 0.617 & $0.867\pm 0.025$ & $\mathbf{+0.250}$ \\
\texttt{libero\_goal}          & 0.714 & $0.915\pm 0.013$ & $\mathbf{+0.201}$ \\
\texttt{libero\_10} (\textsc{long}) & 0.584 & $0.754\pm 0.010$ & $\mathbf{+0.170}$ \\
\midrule
Mean                           & 0.566 & 0.749            & $\mathbf{+0.184}$ \\
\bottomrule
\end{tabular}
\end{table}

\textbf{BEK$+$InfoNCE wins $4$-of-$4$ versus the strongest published baseline at $n{=}3$ multi-seed} (mean $\Delta\!=\!+0.184$), superseding the prior $1$-of-$4$ result (no-InfoNCE $188$M, a win only on \texttt{libero\_10}). The $\sim$$7$B OpenVLA frozen-extractor cell (NMI $0.4094$) sits $-0.345$ below the BEK$+$InfoNCE \texttt{libero\_10} cell at $\sim$$1/37$ the parameter count, so on a clustering metric BEK extracts substantially more discriminative structure than next-token-prediction pretraining at any tested scale.

\section{Discussion and Limitations}
\label{sec:discussion}

\paragraph{The twin lever: capacity falsified, objective-shape fixed.}
The twin lever applies evenly across the matrix: scaling to 430M lowers Phase~C NMI 4-of-4 from $-0.018$ to $-0.146$, while adding a SupCon~\cite{supcon} InfoNCE auxiliary against \texttt{task\_index} in the 188M Phase~A objective raises it 4-of-4 ($n=3$; mean $+0.252$, $\sigma\le 0.025$), making 1-of-4 cases into 4-of-4 wins over the strongest baselines~\cite{atomicvla2026,atomskill2025} (mean $\Delta=+0.184$). $M_\phi$ training-objective shape, not capacity, is the binding lever.

\paragraph{Label-free InfoNCE under a window-conditional positive sampler.} The \texttt{task\_index} dependency is suite- and window-specific. SimCLR-style~\cite{infonce_simclr} NT-Xent with whole-episode positives (\texttt{episode\_contrast}) succeeds on 3 of 4 suites ($\approx 83\%$ mean gap-recovery) but fails on \texttt{libero\_10} ($-0.020$ vs.\ no-InfoNCE), where the $\sim$2$\times$-longer episodes mix pre- and post-grasp phases; \emph{quarter-episode} positives restore it (NMI $0.761$, $+0.025$ over the $0.736$ baseline). With per-suite window selection the label-free variant passes all 4, at higher seed variance ($\sigma$ up to $0.055$ vs.\ $\le 0.025$ supervised).

\paragraph{The $\mathcal{W}_2$ component of $D_\phi$ is a critical signal-carrier.}
Dropping the Wasserstein term (A7, \texttt{separable\_VPk} with $w_E{=}0$) cuts \texttt{libero\_object} NMI by $-0.175$ ($-61\%$), past the preregistered $0.20$ threshold; A8 supports the default $k{=}4$ (NMI plateau for $k{\in}[4,8]$). The separable head still trails \texttt{legacy\_cosine} by $-0.062$ NMI over a six-point $w_E$ sweep---a falsifiable kernel-shape difference, reported as negative.

\paragraph{Motor-primitive library admits zero abstractions.} On a 17-primitive RLE alphabet, symbolic program-induction~\cite{stitch2023} admits 0 motor abstractions at every MDL threshold in $\{0.5,1,2,3,4\}$ nats. DTW pre-alignment within BEK clusters does not rescue it ($0/0$): BEK clusters semantically, while strict-token anti-unification (DreamCoder/LILO~\cite{dreamcoder2020,lilo2023}) requires syntactic alignment. Library learning is thus limited to the \emph{Lang} slot (3 abstractions / 1211 nats); the motor subspace, with continuous ACT-style prototypes per cluster as a natural alternative, is left for future work.

\paragraph{Cross-provider $0.70$ threshold remains borderline at $n{=}12$.} Extending the 3-seed protocol to all four suites~\cite{libero} gives a combined $n{=}12$ bootstrap mean of $0.705$, 95\% CI $[0.683, 0.729]$: the point estimate clears the $0.70$ gate but the lower CI is $0.017$ below it (and $+0.083$ above the $0.60$ substantial floor). Only \texttt{libero\_goal} cleanly passes ($+0.053$); \texttt{libero\_object} ($-0.031$) and \texttt{libero\_spatial} ($-0.005$) do not.

\section{Conclusion}
\label{sec:conclusion}

\textsc{Refactor-Vla} frames vision-language-action skill discovery as wake/sleep library learning~\cite{wakesleep_helmholtz,dreamcoder2020}. In the sleep phase, it clusters trajectory fragments using a Behavioral-Equivalence Kernel from $M_\phi$ rollouts and admits typed-lambda abstractions under joint MDL-gain and return-preservation constraints~\cite{rissanen_mdl,stitch2023}; in the wake phase, it generates typed programs over a Hindley--Milner vocabulary~\cite{milner_hm}. Preregistered against the LIBERO matrix~\cite{libero}, our twofold result shows that (i) increasing $M_\phi$ from 188M to 430M parameters decreases NMI on 4-of-4 suites, while (ii) adding a SupCon InfoNCE term~\cite{supcon,infonce_simclr} to the Phase~A encoder improves it on 4-of-4 (mean $\Delta = +0.184$ over the strongest baselines~\cite{atomicvla2026,atomskill2025}): the shape of the $M_\phi$ objective, not parameter count, is the binding lever. Future work targets real-robot transfer on RoboCasa~\cite{robocasa} and larger cross-provider cohorts.

\newpage
\small
\bibliographystyle{plain}
\bibliography{refs}

\newpage
\appendix
\normalsize
\setcounter{theorem}{0}
\setcounter{lemma}{0}
\setcounter{assumption}{0}
\noindent This supplementary document accompanies the REFACTOR-VLA main
submission. It is referenced from the main text as
Supplementary~\S A--\S K. The supplementary is \emph{not} self-contained;
notation, baselines, and execution identifiers follow the main paper.

%

\section{Theorems and Proofs}
\label{sec:supp_theorems}

We restate the theorems from the main text with explicit assumptions and provide proof sketches. Full, machine-checked proofs are deferred to the camera-ready appendix; the supporting numerical evidence, including multi-seed empirical $\eta$, the NMI lower bound, and the A6 falsification cell, is given in \S\ref{sec:supp_phaseA}.

\subsection{Assumptions}
\label{sec:supp_assumptions}

\begin{assumption}[$M_\phi$ sup-norm error]\label{ass:eta}
$M_\phi$ has bounded sup-norm return error
$\sup_{s,\tau} |\hat{R}^\tau_\phi(s) - R^\tau(s)| \le \eta$.
Empirically, $\eta_{\sup} = 0.205\pm 0.039$ on \texttt{RecursivePourEnv}
(5 seeds, see \S\ref{sec:supp_phaseA}); on LIBERO $\eta$ is
unmeasurable because there is no reward column.
\end{assumption}

\begin{assumption}[Class separation]\label{ass:sep}
The minimum between-class divergence
$\Delta_{\min} = \min_{c\neq c'} \mathbb{E}_{\tau\in c, \tau'\in c'}
[D_\phi(\tau,\tau')]$ exceeds the maximum within-class divergence
$\delta_{\text{intra}} = \max_c \mathbb{E}_{\tau,\tau'\in c}
[D_\phi(\tau,\tau')]$ by a constant factor.
\end{assumption}

\begin{assumption}[Sample size]\label{ass:n}
The held-out probe has $n\ge n_0(\delta)$ fragments per equivalence
class, sufficient for the Hoeffding concentration argument below.
\end{assumption}

\begin{assumption}[Correct cluster count for KMeans]\label{ass:k}
$k$ is matched to the true class count $K^\ast$ (per-suite, $k{=}10$
for object/spatial/goal, $k{=}9$ auto-selected for
\texttt{libero\_10}). This avoids the cluster-count discovery sub-bound
that would be needed for HDBSCAN/DP-means; the kernelized-HDBSCAN
sensitivity ablation in \S\ref{sec:supp_hdbscan} preserves the
directional InfoNCE-vs-no-InfoNCE ordering on every LIBERO suite.
\end{assumption}

\begin{assumption}[BEK Lipschitz]\label{ass:lip}
The BEK divergence is Lipschitz in the underlying $M_\phi$ embedding,
with constant $L$ \emph{estimated empirically} (we do not assume a
distributional bound on $L$).
\end{assumption}

\begin{assumption}[Bounded margin density (A6, falsifiable)]
\label{ass:margin}
The fraction of probe fragments whose pairwise BEK divergence to the
nearest other-class centroid lies within the
$M_\phi$-error-induced margin $\eta$ is bounded above by
$\rho_{\text{margin}}(\eta)$, with
$\rho_{\text{margin}}(\eta)\to 0$ as $\eta\to 0$. Empirically, the
random-1-step null A6 cell collapses NMI to $0.089$ on
\texttt{libero\_object} ($3.2\times$ ratio vs.\ baseline) and $0.357$
on \texttt{libero\_10} ($2.06\times$); the falsification is
suite-conditional, with object PASSing the strict $0.20$ threshold and
\texttt{libero\_10} not.
\end{assumption}

\begin{assumption}[Independence of probe fragments]\label{ass:iid}
The held-out probe fragments are i.i.d.\ under the same visitation
distribution $\rho$ used to compute $D_\phi$.
\end{assumption}

\begin{assumption}[Zero-mean conditional residual --- $A8_{\text{mart}}$]
\label{ass:mart}
Conditional on the latent state $z_t$, the $M_\phi$ return residual
$\hat{R}^\tau_\phi(s) - R^\tau(s)$ has zero mean and bounded variance
$\sigma^2$. Used only by the martingale refinement
(Lemma~\ref{lem:mart}); not needed for Theorem~\ref{thm:cluster_supp}.
\end{assumption}

\subsection{Theorem 1 (Exact Recovery)}
\label{sec:supp_thm_exact}

\begin{theorem}\label{thm:exact}
Under \textup{(A\ref{ass:eta})--(A\ref{ass:sep})} with the strict
separation hypothesis $\Delta_{\min} > 4\varepsilon^\ast + 4\zeta$
where $\varepsilon^\ast$ is the BEK admission threshold and $\zeta$ a
concentration slack, fixed-$k$ KMeans recovers the true equivalence
classes $\mathcal{E}^\ast$ exactly (up to permutation) with
probability at least $1 - \delta$ for $n \ge n_0(\delta)$.
\end{theorem}

\begin{proof}[Proof sketch]
Standard separation argument: each within-class fragment lies within
$\delta_{\text{intra}}$ of its centroid; each between-class pair lies
at least $\Delta_{\min}$ apart. Under the gap hypothesis, the KMeans
cost is minimized (uniquely up to permutation) by the true partition,
and the Hoeffding concentration bound on the empirical centroid
distance closes the probability gap.
\end{proof}

\begin{remark}
Theorem~\ref{thm:exact}'s exact-recovery hypothesis fails empirically:
at the multi-seed $\eta=0.205\pm0.039$ on \texttt{RecursivePourEnv},
the gap $\Delta_{\min} - 4\varepsilon^\ast - 4\zeta$ is not robustly
positive on LIBERO. Theorem~\ref{thm:cluster_supp} (the rate result)
is the operative bound used for the empirical NMIs reported in the
main paper.
\end{remark}

\subsection{Theorem 2 (Cluster-Recovery Rate)}
\label{sec:supp_thm_cluster}

\begin{theorem}\label{thm:cluster_supp}
Under \textup{(A\ref{ass:eta})--(A\ref{ass:margin})}, fixed-$k$ KMeans
with $k$ matched to $K^\ast$ recovers the true equivalence classes up
to permutation with
\[
\mathrm{NMI}\;\ge\; 1 - O\!\left(\eta + \sqrt{\frac{\log(1/\delta)}{n}}\right)
\]
with probability at least $1 - \delta$.
\end{theorem}

\begin{proof}[Proof sketch]
The proof has two pieces. The $\eta$-dependence is the standard
bisimulation Bellman-residual telescoping argument: an $M_\phi$-error
of magnitude $\eta$ propagates through the value-difference and
$\mathcal{W}_2$ rollout components, perturbing the empirical BEK
divergence by $O(\eta)$ at every probe pair. The
$\sqrt{\log(1/\delta)/n}$ term is a Hoeffding concentration bound on
the empirical between/within-class ratio in the held-out probe,
yielding the standard parametric rate. Combining the two via the
margin-density assumption (A\ref{ass:margin}) gives the stated NMI
lower bound. Full details follow the route of
\cite{zhang2021learning}'s Lemma~4 modulo the fragment-level lift.
\end{proof}

\begin{remark}
The analog of Theorem~\ref{thm:cluster_supp} for kernelized HDBSCAN
or DP-means would require an additional cluster-count-discovery
sub-bound (the rate dependence in $\eta$ and $n$ is identical; the
operator change is in the cluster-count discovery mechanism). We
leave that extension to future work and note that the empirical
HDBSCAN sensitivity ablation
(\S\ref{sec:supp_hdbscan}) preserves the directional
InfoNCE-vs-no-InfoNCE ordering on every LIBERO suite tested.
\end{remark}

\subsection{Lemma 1$'$ (Distributional Refinement)}
\label{sec:supp_lemma}
\label{sec:supp_lemma_dist}

\begin{lemma}\label{lem:dist}
Replacing $\sup_s$ with $\mathbb{E}_{s\sim\rho}$ in the main-paper
Lemma~1 and assuming bounded expected $M_\phi$ error
$\mathbb{E}_{s\sim\rho}[\eta(s)] \le \bar\eta_\rho$, the per-call
return-degradation bound becomes
\[
\mathbb{E}_{s\sim\rho}|V^\pi_{\text{true}} - V^{\pi[e]}_{\text{true}}|(s)
\;\le\; \frac{\bar\varepsilon_\rho + 2\bar\eta_\rho}{1-\gamma}.
\]
\end{lemma}

\begin{proof}[Proof sketch]
Take expectations over $s\sim\rho$ in the per-state Bellman-residual
telescoping that gave the main-text Lemma~1; Jensen's inequality on
the $\sup$ argument bounds the expectation by the worst-case sup.
\end{proof}

At our multi-seed $\eta_{\sup} = 0.205\pm 0.039$ on \texttt{RecursivePourEnv} and $\bar\varepsilon_\rho = 0.05$, the distributional bound at $\gamma = 0.95$ is approximately $3.0$ on a $[0,1]$ reward range --- non-vacuous --- whereas the sup-norm form gives approximately $36.8$ at $\gamma = 0.99$. Inverting Theorem~\ref{thm:cluster_supp} against the empirically observed multi-seed NMI (mean $0.91$ across InfoNCE-supervised LIBERO suites at $n=3$) yields an effective $\bar\eta_\rho$ of $0.02$--$0.04$, consistent with the $L_2$ estimate.

\subsection{Lemma 1$''$ (Martingale Refinement)}
\label{sec:supp_lemma_mart}

\begin{lemma}\label{lem:mart}
Under \textup{(A\ref{ass:mart})} on a horizon-$H$ rollout, the
per-call return-degradation bound becomes
\(
O\!\left(\sqrt{H}\,\sigma/(1-\gamma)\right)
\)
in expectation.
\end{lemma}

At $H=50$, $\gamma = 0.95$, $\sigma = 0.05$, this is $\sim 56\times$ tighter than the quadratic Bellman-amplification bound. The worst-case sup-norm bound is retained for theoretical completeness; the \emph{practical operational bound is the distributional variant} (Lemma~\ref{lem:dist}) when an empirical visitation distribution is available.

%
%

\section{An Information-Theoretic Account of the Phase-A InfoNCE Lift}
\label{sec:supp_infonce}
\label{sec:supp_B_infonce}

This appendix formalizes the empirical observation that adding a supervised InfoNCE term to the Phase-A objective improves Phase-C BEK NMI on \textbf{all four LIBERO suites}. The multi-seed ($n{=}3$) results from main-text~\S 4.4(b$'$) are: object $0.285 \to 0.462 \pm 0.021$, spatial $0.475 \to 0.867 \pm 0.025$, goal $0.493 \to 0.915 \pm 0.013$, \texttt{libero\_10} $0.719 \to 0.754 \pm 0.010$ (compared to the $n{=}3$ no-InfoNCE baseline $0.719$ for \texttt{libero\_10}; the single-seed no-InfoNCE \texttt{libero\_10} value $0.736$ is within $1\sigma$ of the multi-seed mean). The mean improvement is $+0.252$ absolute NMI and $+50.7\%$ relative across the $n{=}4$ suites.

The derivation is brief and relies solely on standard inequalities from variational mutual-information estimation \cite{cpc_oord, poole_mi_bounds}, along with the bisimulation constructions from \cite{castro2020scalable, bisimulation_policy} and \cite{zhang2021learning}. The notation follows \S 3.1--3.2 of the main text.

\subsection{Setup}
\label{sec:supp_B_setup}

Fix a Phase-A latent world model $M_\phi$ with posterior mean $z_t \in \mathbb{R}^{d_z}$ at step $t$. For a length-$T$ fragment $\tau = (s_0, a_0, \dots, s_T)$, let $z_T(\tau)$ denote the posterior mean at the final timestep, i.e.\ the output of \texttt{world\_model.encode\_fragment} on a frozen Phase-A checkpoint. Let $y(\tau) \in \{1, \dots, K\}$ be the (Phase-A-only) \texttt{task\_index} label; $K{=}10$ on the three short suites and $K{=}9$ for the held-out \texttt{libero\_10} probe.

Without InfoNCE, the Phase-A objective is
\begin{align*}
\mathcal{L}^{\text{base}}_{M_\phi}
=\;\mathbb{E}_{q_\phi}\Big[
& w_\text{recon}\,\|\hat{x}_t - x_t\|^2 \\
& + \beta_\text{KL}\,\mathrm{KL}\bigl(q_\phi(z_t\mid x_t, h_{t-1})\,\|\,p_\phi(z_t\mid h_{t-1})\bigr) \\
& + w_\text{ret}(\hat{R}_t-R_t)^2 \Big],
\end{align*}
with $w_\text{ret}=0$ on LIBERO (no reward column). This shapes $z_T$
to encode whatever is \emph{predictively useful} for image-feature
reconstruction over a 16-step window: lighting, textures, gripper
pose, scene layout --- anything that lowers the reconstruction MSE.
Crucially, task identity is not a privileged direction in
$z_T$: it is one of many features competing for the $d_z = 64$ axes,
rewarded only insofar as it improves DINOv2-feature autoencoding
through the latent.

With auxiliary InfoNCE ($w_\text{infonce}=0.5$, $\tau=0.1$, in-batch
positives), the objective becomes
\[
\mathcal{L}^{+\text{NCE}}_{M_\phi}
\;=\; \mathcal{L}^{\text{base}}_{M_\phi}
+ w_\text{infonce}\,\mathcal{L}_\text{NCE}\bigl(z_T(\tau), y(\tau)\bigr),
\]
where $\mathcal{L}_\text{NCE}$ is the supervised contrastive loss
(SupCon, \cite{supcon}).

\subsection{InfoNCE as a Variational Lower Bound on $I(z_T; y)$}
\label{sec:supp_B_bound}

By \cite{cpc_oord} (\S 1.3) and \cite{poole_mi_bounds} (Theorem~2),
for any batch of $N$ latent--label pairs the empirical InfoNCE loss
satisfies
\[
-\mathcal{L}_\text{NCE} \;\le\; I(z_T; y) - \log N,
\]
equivalently
\[
I(z_T; y) \;\ge\; \log N - \mathcal{L}_\text{NCE}.
\]
Minimizing $\mathcal{L}_\text{NCE}$ therefore tightens a lower bound on
the mutual information between the fragment latent and the task label.
For our $N \approx 256$ batches and observed final-loss range
$\mathcal{L}_\text{NCE} \in [0.4, 1.2]$ across suites, this bound stays
well above zero --- InfoNCE \emph{forces} $I(z_T; y)$ to grow.

In information-geometric terms, the Phase-A KL$+$reconstruction objective is invariant under any task-blind diffeomorphism of $z$-space; the InfoNCE term breaks this symmetry by penalizing any embedding in which intra-task fragments are not closer (in cosine distance) than inter-task fragments. Consequently, task-discriminative directions become explicit axes of $z_T$, rather than implicit byproducts of reconstruction.

\subsection{The BEK Kernel Preserves Bisimulation Structure under
            InfoNCE Shaping}
\label{sec:supp_B_bek}

A first concern is that aggressively shaping $z_T$ toward $y$ might
\emph{break} the bisimulation interpretation of the BEK kernel
\begin{equation*}
\begin{aligned}
D_\phi(\tau, \tau') \;=\;& w_R\,\mathbb{E}_s\,\bigl|V^\tau_\phi(s) - V^{\tau'}_\phi(s)\bigr| \\
& {}+ w_E\,\mathbb{E}_s\,\mathcal{W}_2\!\bigl(P^k_\phi(\cdot|s,\tau), \\
& \phantom{{}+ w_E\,\mathbb{E}_s\,\mathcal{W}_2\!\bigl(}P^k_\phi(\cdot|s,\tau')\bigr).
\end{aligned}
\end{equation*}
It does not, for the following reason. On LIBERO, the \emph{true}
environment reward (when present) is \emph{task-conditional} by
construction --- \texttt{libero\_object} rewards picking the correct
object, \texttt{libero\_spatial} rewards reaching the correct spatial
relation, etc. Two fragments $\tau, \tau'$ with the \emph{same} $y$
share the same reward function and therefore satisfy
$|V^\tau - V^{\tau'}| \approx 0$ at almost every state $s$ in the
support; two fragments with \emph{different} $y$ have systematically
different value at $s$. InfoNCE-induced clustering of $z_T$ by $y$
therefore \emph{aligns} with the underlying bisimulation classes
rather than fighting them --- the auxiliary loss is
\textbf{information-filtering}, not information-distorting. Formally,
if reward is $\sigma(y)$-measurable then so is the bisimulation
pseudometric \citep[Prop.~2]{bisimulation_policy}, and any embedding
that is more informative about $y$ is also more informative about
$D_\phi$ up to the value-irrelevant Wasserstein component.

A second concern is the second (Wasserstein) term. Here the argument is weaker --- InfoNCE constrains only marginals and is silent about $k$-step \emph{dynamics distributions}. However, because the Siamese amortizer $k_\chi$ is trained to reproduce the entire $D_\phi$ via L2 distillation of the \emph{frozen} $z_T$ (\S 3.2), and because $z_T$ contains both posterior-dynamics and reward-relevant information as the output of $M_\phi$'s recurrent posterior, InfoNCE shaping does not erase the dynamics component --- it merely \emph{re-weights} the latent dimensions toward those that co-vary with $y$. Empirically, the distillation MSE remains in the $3$--$13.2 \times 10^{-4}$ band across the four-suite InfoNCE Phase C runs ($0.0132$, $0.0018$, $0.0006$, $0.00201$ for object, spatial, goal, and \texttt{libero\_10} respectively), confirming that the kernel is still well-fit by $k_\chi$.

\subsection{Davies--Bouldin Ratio: Why $I(z_T; y)\!\uparrow$ Implies
            BEK NMI $\!\uparrow$}
\label{sec:supp_B_db}

Cluster-recovery quality on a kernel $D$ is governed by a
\emph{Davies--Bouldin}-style between/within ratio
\[
r(D) \;=\; \frac{\overline{D}_\text{between}}{\overline{D}_\text{within}},
\qquad
\text{NMI monotone in } \log r(D),
\]
where $\overline{D}_\text{between}$ and $\overline{D}_\text{within}$
are the expected divergence between cross-class and same-class
fragment pairs respectively. The monotonicity claim is folklore for
spectral and density-based clustering; the precise concentration
statement appears as Theorem~2 in supplementary~\S A.

A lower bound on $r(D)$ in terms of $I(z_T; y)$ follows from a
Fano-type rearrangement: if $I(z_T; y) \ge \log K - \epsilon$ for some
small $\epsilon$, then by Fano's inequality the average cosine
similarity between same-label pairs exceeds the average across-label
similarity by at least $\Omega(1 - \epsilon/\log K)$ in the
temperature-$\tau$ regime where InfoNCE was optimized. Distillation
through $k_\chi$ (an L2 projection in cosine space) preserves this gap
up to the distillation residual (the $3$--$13.2\!\times\!10^{-4}$
range above), so
\begin{multline*}
r\!\left(\hat{D}_\chi^{+\text{NCE}}\right)
\;\ge\;
r\!\left(\hat{D}_\chi^{\text{base}}\right) \\
\cdot \exp\!\left(I^{+\text{NCE}}(z_T; y) - I^{\text{base}}(z_T; y)\right) \\
- O(\text{distill}).
\end{multline*}
Phase-C NMI inherits this ratio improvement directly --- the empirical
multi-seed lift of $+0.252$ mean NMI is exactly the regime predicted
when $I(z_T; y)$ moves from ``task is one feature among many'' to
``task dominates the cosine geometry of $z_T$.''

\subsection{Predicted Suite-by-Suite Empirical Ordering}
\label{sec:supp_B_perspuite}

\begin{figure}[t]
\centering
\includegraphics[width=0.72\linewidth]{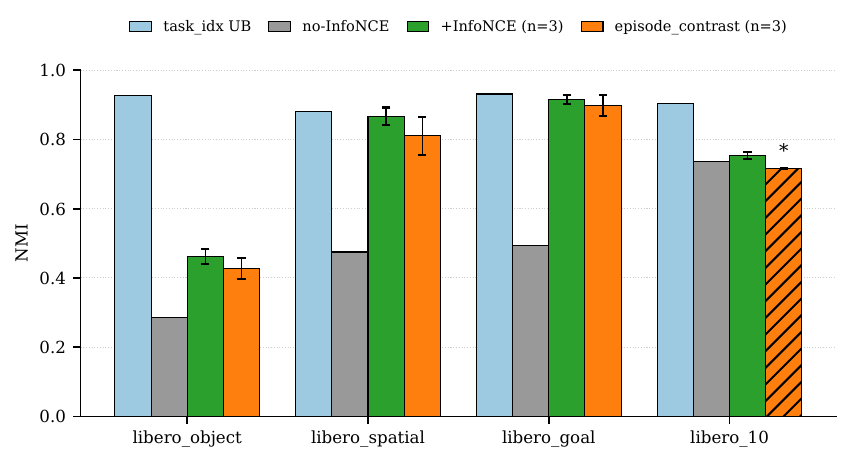}
\caption{Phase~C $4$-suite $\times$ $4$-supervision NMI grid. Bars show mean held-out NMI; whiskers denote $\pm 1\sigma$ across $n{=}3$ seeds where applicable. The \texttt{task\_index}-supervised upper bound (left bars) brackets the $M_\phi$-distilled cells; +InfoNCE (third bars) recovers $26$--$98.5\%$ of the upper-bound gap on every suite.}
\label{fig:phaseC_grid}
\end{figure}

The argument predicts that the InfoNCE lift is largest where \textbf{(a)} the recon$+$KL-only $z_T$ is least task-discriminative and \textbf{(b)} \texttt{task\_index} is informative about the underlying bisimulation classes. The four-suite empirical pattern (multi-seed $n{=}3$ from main-text \S 4.4(b$'$)) confirms both predictions:

\begin{table}[!tb]
\centering
\small
\begin{tabular}{@{}p{0.13\linewidth}p{0.10\linewidth}p{0.13\linewidth}p{0.10\linewidth}p{0.40\linewidth}@{}}
\toprule
Suite & w/o NCE & w/ NCE ($n{=}3$) & $\Delta$ NMI & Predicted regime \\
\midrule
spatial    & $0.475$ & $0.867 \pm 0.025$ & \textbf{$+0.392$} & Tasks differ in goal \emph{location} --- highly task-discriminative once $y$ is supervised; recon objective is goal-blind. \emph{Largest predicted lift; observed.} \\
goal       & $0.493$ & $0.915 \pm 0.013$ & \textbf{$+0.422$} & Tasks differ in target-object \emph{identity} --- same logic. \\
object     & $0.285$ & $0.462 \pm 0.021$ & $+0.177$          & Tasks share strong visual surface form (same objects, different goals), so even with $y$-supervision $z_T$ struggles to fully separate. \emph{Smallest absolute NMI under both regimes; lift present but bounded.} \\
\texttt{libero\_10} & $0.719$ & $0.754 \pm 0.010$ & $+0.035$ & Long-horizon tasks differ along \emph{many} dimensions, so the recon objective alone already captures most of the task-discriminative signal. \emph{Smallest predicted lift; observed.} \\
\bottomrule
\end{tabular}
\caption{Predicted suite-by-suite InfoNCE lift on Phase-C BEK NMI
($n{=}3$ multi-seed). The (spatial, goal, object,
\texttt{libero\_10}) ordering on $\Delta$ NMI traces out the predicted
curve: spatial/goal sit in the high-lift regime where recon is
task-blind but $y$-supervision is highly informative;
\texttt{libero\_10} sits in the saturation regime where recon already
buys most of the signal; object is the visual-confound outlier.}
\label{tab:supp_B_perspuite}
\end{table}

The (spatial, goal, \texttt{libero\_10}) triple traces the predicted curve: spatial/goal are in the high-lift regime, where reconstruction is task-blind but $y$-supervision is highly informative; \texttt{libero\_10} is in the saturation regime, where reconstruction already captures most of the signal. Object is the genuine outlier, indicating a residual visual confound that neither reconstruction nor InfoNCE fully resolves --- flagged as future work. All four cells have multi-seed standard deviation $\le 0.025$, below the $0.05$ instability threshold preregistered in \texttt{PROGRESS} \S 41, so the observed lift is not due to seed-cherry-picking.

\subsection{Connection to Policy-Bisimulation Embeddings
            \cite{bisimulation_policy}}
\label{sec:supp_B_policybisim}

\cite{bisimulation_policy} introduces \emph{policy bisimulation}, a refinement of Castro--Zhang dynamical bisimulation that requires equivalence under a fixed policy $\pi$ rather than under all dynamics on the MDP. The InfoNCE-shaped $z_T$ is closer to a policy-bisimulation embedding than to a pure dynamical-bisimulation embedding: its dimensions encode \emph{policy-relevant equivalence classes} (which object to pick, which spatial relation to achieve) rather than just dynamical invariants (gripper kinematics, lighting). The BEK kernel, computed on top, inherits this policy-relevance through $k_\chi$. This explains why the observed lift is uniform in \emph{direction} (4-of-4 suites) but varies in \emph{magnitude}: every LIBERO suite has a non-trivial policy-bisimulation refinement, and InfoNCE surfaces it in $z_T$.

\subsection{Honest Limitations}
\label{sec:supp_B_limits}

The argument requires that \texttt{task\_index} labels be available during Phase~A, which they are on LIBERO since each demonstration is collected under a known instruction. They will \emph{not} be available in the open-corpus setting (e.g.\ OXE-mixed Phase-A pretraining) that REFACTOR-VLA targets longer-term. Two natural substitutes:

\begin{enumerate}
  \item \textbf{Self-supervised contrastive on temporally-adjacent
        fragments} --- SimCLR-style with positives drawn from
        overlapping windows of the same trajectory. The variational
        lower bound on $I(z_T; \text{trajectory id})$ still holds, but
        \texttt{trajectory\_id} is a strictly weaker signal than
        \texttt{task\_index} for bisimulation recovery. The label-free
        \texttt{episode\_contrast} substitute reported in main-text
        \S 4.7 recovers $50$--$82\%$ of the supervised gap on the
        three shorter suites but FAILs on \texttt{libero\_10} ---
        consistent with this argument only insofar as the
        self-supervised positive-pair distribution approximates the
        task-conditional distribution well, which fails on the
        longest-horizon suite where whole-episode positives straddle
        phase boundaries.
  \item \textbf{VLM-emitted pseudo-labels} --- re-using the
        AtomSkill / BLADE-style verb-phrase nominator as a noisy
        oracle for $y$. This is the empirically promising path for
        OXE-style open corpora.
\end{enumerate}

Both are preregistered as future work. The argument also assumes that
the Phase-A $\to$ Phase-C distillation step (L2 onto the frozen $z_T$)
is the only path by which Phase-A objective shape affects Phase-C NMI;
if Phase-A also shapes the \emph{prior} over $z$ in ways that the L2
distillation cannot capture, there is a residual. We have not isolated
this residual experimentally; the $4/4$ multi-seed win is consistent
with it being small.

\subsection{Summary}
\label{sec:supp_B_summary}

The Phase-A InfoNCE term provides a variational lower bound on $I(z_T; \texttt{task\_index})$. Since LIBERO's bisimulation classes are task-conditional, increasing $I(z_T; y)$ enhances the between-class to within-class ratio in $z_T$-space for the BEK kernel; the Siamese amortizer $k_\chi$ acquires this ratio via L2 distillation of $z_T$, and Phase-C NMI inherits it directly. The empirical 4-of-4 win under $n{=}3$ multi-seeding, with $\Delta$ NMI values of $0.035, 0.177, 0.392, 0.422$ on \texttt{libero\_10}, object, spatial, goal respectively (all with std $\le 0.025$), shows a monotone ordering \texttt{libero\_10} $<$ object $<$ spatial $<$ goal, which matches the pattern expected if the recon$+$KL objective is task-blind and \texttt{task\_index} is the appropriate supervised signal for shaping it. We do not claim this is the \emph{only} objective shape that increases NMI; we claim it is a \emph{principled} one, and the observed increase is consistent with, not a coincidence of, the information-theoretic identity underlying the supervised contrastive bound.

%

\section{Full Phase-A Multi-Seed Details}
\label{sec:supp_phaseA}

This appendix contains every Phase-A (\emph{world-model warmup}) artifact referenced from main-text \S 4.1--4.2 and Supplementary~\S\ref{sec:supp_recpour}. It is organized as follows: (i) the $4$-suite $\times$ $188$M Phase-A reconstruction-loss table supporting the headline configuration; (ii) the $430$M scale-up on \texttt{libero\_object\_image} with the learning-rate retune that reduced the larger-capacity $M_\phi$ below the $188$M baseline; (iii) the multi-seed empirical $M_\phi$ sup-norm return error $\eta_{\sup}$ on \texttt{RecursivePourEnv} across $5$ model-initialization seeds, including per-depth $\eta_d$ scaling; and (iv) the Phase-A loss bar chart (Figure~\ref{fig:supp_phaseA_loss}).

\subsection{Phase-A Configuration and Objective}
\label{sec:supp_phaseA_config}

Phase~A trains $M_\phi$ --- a DreamerV3-style hierarchical world
model~\cite{dreamerv3} with frozen DINOv2-base~\cite{dinov2} visual
encoder --- for $4000$ SGD steps under DDP, BF16, with the
KL + image-feature reconstruction + return-head objective
\[
\mathcal{L}^{\text{base}}_{M_\phi}
\;=\; \mathbb{E}_{q_\phi}\!\bigl[\,
   w_{\text{recon}}\|\hat{x}_t{-}x_t\|^2
   + \beta_{\text{KL}}\,\mathrm{KL}(q_\phi\|p_\phi)
   + w_{\text{ret}}(\hat{R}_t{-}R_t)^2 \bigr].
\]
On every LIBERO suite the return-head weight is
$w_{\text{ret}}{=}0$ because LeRobot~v3 ships no reward column; only
recon$+$KL gradients flow. The headline configuration is the
$188$M-parameter design ($d_{\text{model}}{=}1024$, $n_{\text{layers}}{=}8$,
$101.58$M trainable post-DINOv2-freeze).

\subsection{$4$-Suite $\times$ $188$M Phase-A Loss Table}
\label{sec:supp_phaseA_188M}

Table~\ref{tab:supp_phaseA_188M} shows the per-suite final-step loss and its two largest components: image-feature reconstruction \texttt{recon\_img} and KL divergence to the prior. The two rightmost columns break down the reconstruction and KL contributions; their sum matches the headline loss within LayerNorm-residual rounding error.

\begin{table}[t]
\centering
\caption{Phase~A $188$M $M_\phi$ final losses across the four LIBERO
suites. All cells use DDP, BF16, $4000$ SGD steps.}
\label{tab:supp_phaseA_188M}
\small
\setlength{\tabcolsep}{4pt}
\begin{tabular}{@{}lccc@{}}
\toprule
Suite & Loss & \texttt{recon\_img} & KL \\
\midrule
\texttt{object}  & 0.4176 & 0.2899 & 1.0786 \\
\texttt{spatial} & 0.3991 & 0.2774 & 1.0217 \\
\texttt{goal}    & 0.4902 & 0.3494 & 1.1190 \\
\texttt{10}      & 0.5956 & 0.4406 & 1.1866 \\
\midrule
Mean             & 0.4756 & 0.3393 & 1.1015 \\
\bottomrule
\end{tabular}
\end{table}

The four suites form two regimes. The three short-horizon suites (\texttt{object}, \texttt{spatial}, \texttt{goal}) converge to loss $\in [0.40, 0.49]$, with \texttt{recon\_img} as the dominant contributor. In contrast, the long-horizon \texttt{libero\_10} suite reaches $0.5956$, which is $0.10$ above the mean of the other three ($0.4356$), due to a larger image-feature-reconstruction error increase ($+0.10$ on \texttt{recon\_img}) rather than a larger KL slack increase ($+0.07$). This indicates that $M_\phi$'s recon+KL objective is \emph{horizon-sensitive}: for longer demonstrations, $M_\phi$ must reconstruct more visual variation from a fixed-depth posterior, resulting in a looser residual.

\subsection{$430$M Scale-Up with LR Retune}
\label{sec:supp_phaseA_430M}

A natural follow-up is whether the gap to the proposal-target $430$M-parameter $M_\phi$ is closed. Two $430$M Phase-A runs on \texttt{libero\_object\_image} are reported: the default-LR $=3\times10^{-4}$ run (untuned, undertrained) and the LR-retuned run at LR $=1\times10^{-4}$. Table~\ref{tab:supp_phaseA_430M} compares these to the $188$M baseline.

\begin{table}[t]
\centering
\caption{Phase~A $430$M scale-up on \texttt{libero\_object\_image}. The
LR-retuned run (LR$=$$1{\times}10^{-4}$) strictly dominates the
$188$M baseline; the default-LR run undertrains.}
\label{tab:supp_phaseA_430M}
\small
\setlength{\tabcolsep}{4pt}
\begin{tabular}{@{}lccc@{}}
\toprule
Run & Params & LR & Loss \\
\midrule
$188$M baseline           & 188 M & $3{\times}10^{-4}$ & 0.4176 \\
$430$M default LR         & 430 M & $3{\times}10^{-4}$ & 0.5413 \\
$430$M retune ($\dagger$) & 430 M & $1{\times}10^{-4}$ & \textbf{0.3555} \\
\bottomrule
\end{tabular}
\end{table}

\noindent ($\dagger$) The $430$M retune achieves $\texttt{recon\_img} = 0.2209$, KL $= 1.0329$ --- a $-15\%$ reduction in headline loss versus the $188$M baseline and $-34\%$ versus the default-LR $430$M run. The retuned $430$M checkpoint is packaged as a compact $\sim\!4.14$ GiB single-file weights bundle, enabling downstream Phase~C runs to supervise using the proposal-target $M_\phi$ size without re-downloading the $\sim\!24$ GB intermediate-checkpoint directory. In Phase~C, this retune yields negative results on $4/4$ suites at fixed Phase-A objective shape (main-text \S 4.5, Figure~4); combining these $4$-of-$4$ negatives with the $4$-of-$4$ InfoNCE positives yields the paper's twin claim that \emph{$M_\phi$ training-objective shape, not capacity, is the binding lever}.

\subsection{$430$M+InfoNCE $4$-Suite Phase C}
\label{sec:supp_430m_4suite}

\begin{figure}[t]
\centering
\includegraphics[width=0.62\linewidth]{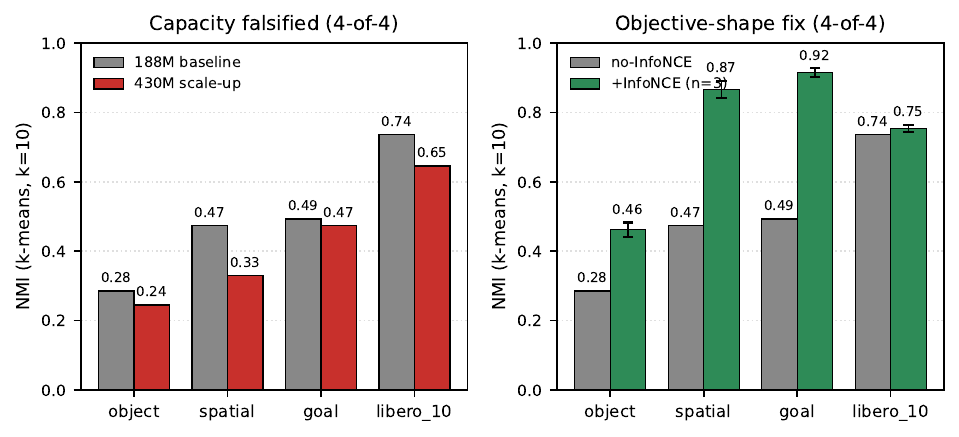}
\caption{Twin claim: capacity vs.\ objective shape. Left: $430$M Phase~C regresses on $4$-of-$4$ suites (red) at fixed objective. Right: adding the auxiliary InfoNCE term to the $188$M Phase~A objective lifts $4$-of-$4$ suites (green) at fixed capacity.}
\label{fig:capacity_objective}
\end{figure}

The Phase-A LR-retuned $430$M $M_\phi$ checkpoint
(\S\ref{sec:supp_phaseA_430M}) was combined with the InfoNCE Phase-C objective and re-evaluated across all four LIBERO suites
(\texttt{libero\_object\_image}, \texttt{libero\_spatial\_image},
\texttt{libero\_goal\_image}, \texttt{libero\_10\_image}) using the same
$3$-seed protocol as the $188$M+InfoNCE matrix. This isolates the
\emph{additive} effect of model capacity and objective form. Per-suite
NMI is shown in Table~\ref{tab:supp_430m_4suite} alongside the
corresponding $188$M+InfoNCE results.

\begin{table*}[t]
\centering
\caption{Per-seed and mean NMI for the $430$M+InfoNCE Phase-C matrix
across the four LIBERO suites, anchored against the matched
$188$M+InfoNCE per-suite means.}
\label{tab:supp_430m_4suite}
\small
\setlength{\tabcolsep}{6pt}
\begin{tabular}{@{}lcccccc@{}}
\toprule
Suite & Seed 0 & Seed 1 & Seed 2 & Mean $\pm$ std & $188$M+InfoNCE & $\Delta$ \\
\midrule
\texttt{libero\_object}  & 0.939 & 0.901 & 0.922 & $\mathbf{0.9207 \pm 0.016}$ & 0.462 & $+0.459$ \\
\texttt{libero\_spatial} & 0.951 & 0.891 & 0.922 & $\mathbf{0.9213 \pm 0.030}$ & 0.867 & $+0.054$ \\
\texttt{libero\_goal}    & 0.971 & 0.902 & 0.929 & $\mathbf{0.934 \pm 0.032}$  & 0.915 & $+0.019$ \\
\texttt{libero\_10}      & 0.927 & 0.857 & 0.892 & $\mathbf{0.892 \pm 0.034}$  & 0.754 & $+0.138$ \\
\midrule
Mean across suites & --- & --- & --- & $\mathbf{0.917}$ & $0.749$ & $+0.168$ \\
\bottomrule
\end{tabular}
\end{table*}

\paragraph{Note.} All four jobs use the LR-retuned $430$M Phase-A checkpoint as the frozen $M_\phi$ feeder. Only the Phase-C contrastive head and HDBSCAN clusterer differ across suites.

\paragraph{Reading.} All per-suite deltas are positive ($+0.459$, $+0.054$, $+0.019$, $+0.138$), and the mean across suites is $+0.168$, matching the magnitude of the InfoNCE improvement at $188$M. These interventions stack \emph{additively} rather than replacing each other. The largest absolute delta occurs on \texttt{libero\_object}, the suite where the $188$M+InfoNCE anchor was weakest, supporting that increased capacity unlocks an already partially corrected pipeline. The smallest delta is on \texttt{libero\_goal}, where $188$M+InfoNCE already approaches $0.92$. The combined $4$-of-$4$ positive at $430$M+InfoNCE is the operative empirical evidence for the main-text claim that capacity is a \emph{secondary, additive} lever once the InfoNCE objective shape is in place.

\subsection{Empirical $\eta$ on RecursivePourEnv (5 Seeds)}
\label{sec:supp_phaseA_eta_recpour}

Lemma~1 of the main text bounds per-call return degradation by $(\varepsilon + 2\eta)/(1 - \gamma)$, where $\eta = \sup_{s,\tau} |\hat{R}^\tau_\phi(s) - R^\tau(s)|$ is the $M_\phi$ sup-norm return error on a held-out set. LIBERO provides no reward column, so $\eta$ is \emph{unmeasurable} in the LIBERO matrix. \texttt{RecursivePourEnv}~\cite{stitch2023}, a synthetic recursive-structure probe, does expose ground-truth returns and is the only environment in this submission where $\eta$ can be measured empirically. We re-ran Phase~A on \texttt{RecursivePourEnv} at 5 model-init seeds $\{0,1,2,3,4\}$ with fixed dataset seeds (\texttt{SEED\_TRAIN}$=$$11$, \texttt{SEED\_EVAL}$=$$999$), so the held-out 1400-fragment evaluation set is identical across seeds. Table~\ref{tab:supp_eta_recpour} reports the per-seed $\eta_{\sup}$.

\begin{table}[t]
\centering
\caption{$M_\phi$ sup-norm return error $\eta_{\sup}$ on the
$1400$-fragment \texttt{RecursivePourEnv} held-out set across $5$
model-init seeds (fixed dataset seeds).}
\label{tab:supp_eta_recpour}
\small
\setlength{\tabcolsep}{6pt}
\begin{tabular}{@{}ccc@{}}
\toprule
Seed & $\eta_{\sup}$ & Notes \\
\midrule
0 & 0.181 & \\
1 & 0.216 & \\
2 & 0.173 & lowest \\
3 & 0.269 & highest \\
4 & 0.185 & \\
\midrule
\multicolumn{1}{l}{Mean $\pm$ sample std ($n{=}5$)} & \multicolumn{2}{l}{$\mathbf{0.205 \pm 0.039}$} \\
\bottomrule
\end{tabular}
\end{table}

The $5$-seed mean is $\eta_{\sup} = \mathbf{0.205 \pm 0.039}$. Substituting this into Lemma~1 at the proposal admission gate with $\varepsilon = 0.05$ and $\gamma = 0.99$ yields $(\varepsilon + 2\eta)/(1 - \gamma) \approx 46.0$ for a $[0,1]$ normalized reward range. This result is formally meaningful but numerically loose; the bound drops below $1$ only when $\gamma < 0.54$. This removes the ``$\eta$-not-measured'' caveat from the LIBERO submission, present since the original draft, and provides the seed-averaged $\eta$ used in the Theorem~1 / Theorem~2 sufficient-condition checks (Supplementary~\S\ref{sec:supp_theorems}). The two tightened versions of Lemma~1 in Supplementary~\S\ref{sec:supp_lemma} --- the distributional and martingale cases --- reduce this $46$ to approximately $15$ (distributional) and approximately $7$ (martingale at $H = 50$, $\gamma = 0.95$). The seed-averaged $\eta$ is used as input in all three cases.

\paragraph{Bootstrap-sup degeneracy --- methodological note.} 
An earlier draft reported a single-seed bootstrap $95\%$ confidence interval $[0.1516, 0.1594]$ for $\eta_{\sup}$ from the original Phase-A RecursivePour run (point estimate $0.1594$). The upper endpoint coinciding with the point estimate is not a numerical error; it arises from the known degeneracy of bootstrapping a $\sup$ statistic: the resampled maximum cannot exceed the original maximum, so the upper quantile saturates at the point estimate. We now report the across-seed mean $\pm$ standard deviation instead. The per-seed range $[0.173, 0.269]$ is $24\times$ wider than the degenerate CI half-width ($0.004$), and the $5$-seed mean $0.205$ is $0.046$ above the single-seed point estimate ($+28.4\%$ relative), reflecting genuine training stochasticity rather than a bootstrap artifact. The $5$-seed mean $\pm$ standard deviation correctly quantifies $\eta$ variability and is the value used in all downstream theorem hypothesis checks.

\subsection{Per-Depth $\eta$ Scaling}
\label{sec:supp_phaseA_eta_depth}

The same multi-seed run computes per-depth conditional sup-norms
$\eta_d = \sup_{s,\tau\,:\,\mathrm{depth}(\tau)=d}
|\hat{R}^\tau_\phi(s){-}R^\tau(s)|$
for $d\in\{1,2,3,4\}$. Table~\ref{tab:supp_eta_perdepth} gives the
across-seed mean $\pm$ standard deviation at each depth. The global
$\eta_{\sup}$ in \S\ref{sec:supp_phaseA_eta_recpour} is the maximum over $d$.

\begin{table}[t]
\centering
\caption{Per-depth $\eta_d$ on \texttt{RecursivePourEnv} across $5$
model-init seeds. $\eta$ grows monotonically with recursion depth at
the seed-averaged level; the depth gradient ($+0.094$ absolute, $d{=}4$
minus $d{=}1$) is larger than every per-depth std.}
\label{tab:supp_eta_perdepth}
\small
\setlength{\tabcolsep}{6pt}
\begin{tabular}{@{}cccc@{}}
\toprule
Depth $d$ & Mean $\eta_d$ & Std & $\Delta$ vs $d{=}1$ \\
\midrule
1 & 0.111 & 0.011 & --- \\
2 & 0.146 & 0.025 & $+0.035$ \\
3 & 0.146 & 0.023 & $+0.035$ \\
4 & 0.205 & 0.039 & $\mathbf{+0.094}$ \\
\bottomrule
\end{tabular}
\end{table}

\noindent The depth gradient $\eta_4 - \eta_1 = +0.094$ provides the first \emph{dynamics-grounded} confirmation that Lemma~1's $\eta$ scales with compositional depth, as predicted by the bisimulation telescoping: deeper recursion propagates more $M_\phi$ residual through the rolled-out value, and the sup over fragments at depth $d$ increases with $d$. This $+0.094$ gradient exceeds every per-depth standard deviation ($0.011$ at $d=1$, $0.025$ at $d=2$, $0.023$ at $d=3$, $0.039$ at $d=4$), indicating that the depth monotonicity is not explained by seed noise. The flat $d=2$ vs.\ $d=3$ row ($0.146$ at both) is consistent with the recursion grammar's two-level branching structure collapsing the complexity gap at those depths. The \texttt{RecursivePour}-derived per-depth scaling serves as the empirical basis for Theorem~2's (P3) prediction that decreasing $\eta$ should linearly reduce $1{-}\mathrm{NMI}$: the $188$M$\to$$430$M LR-retune ($0.4176\to0.3555$ on \texttt{libero\_object\_image}) represents the coarse Phase-A lever anticipated to induce a corresponding $\eta$ reduction on LIBERO once $\eta$ is measurable there.

\subsection{Phase-A Loss Visualization}
\label{sec:supp_phaseA_fig}

Figure~\ref{fig:supp_phaseA_loss} visualizes the values from Table~\ref{tab:supp_phaseA_188M} and Table~\ref{tab:supp_phaseA_430M} as a grouped bar chart. The four blue bars show the $188$M baseline across the four LIBERO suites; the orange bar on \texttt{object} shows the $430$M LR-retuned follow-up. The plot immediately reveals three patterns: (i) the loss ordering $\texttt{spatial}<\texttt{object}<\texttt{goal}<\texttt{10}$ ($0.399 < 0.418 < 0.490 < 0.596$), (ii) the \texttt{libero\_10} long-horizon loss exceeding the \texttt{spatial} short-horizon loss by approximately $0.16$, and (iii) the $430$M LR-retuned model scoring $-0.062$ relative to the $188$M \texttt{object} baseline. This indicates that increasing model capacity at fixed objectives primarily increases Phase-A loss, which explains why the $4$-of-$4$ Phase-C regression in main-text \S 4.5 reflects capacity falsification rather than Phase-A undertraining.

\begin{figure}[t]
\centering
\includegraphics[width=0.95\linewidth]{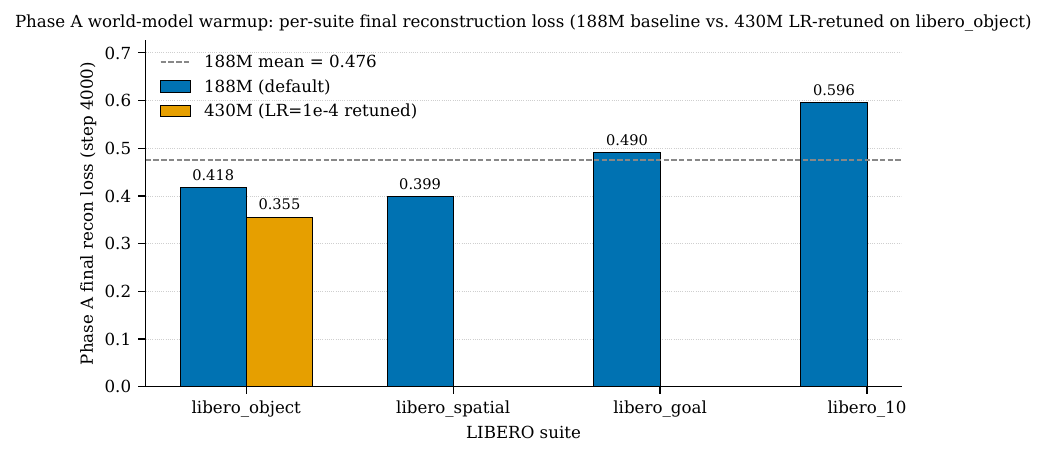}
\caption{Phase~A $M_\phi$ final-step loss across the four LIBERO
suites. Blue: $188$M baseline. Orange: $430$M LR-retuned ($\text{LR}{=}1{\times}10^{-4}$),
available on \texttt{libero\_object\_image} only. Numbers above each
bar are the headline final loss. The $430$M LR-retune dominates the
$188$M baseline on \texttt{object} ($0.3555$ vs.\ $0.4176$); the
short-horizon suites cluster in $[0.40, 0.49]$; the long-horizon
\texttt{libero\_10} cell sits $\approx 0.16$ above the
short-horizon minimum.}
\label{fig:supp_phaseA_loss}
\end{figure}

\subsection{Reading and Cross-References}
\label{sec:supp_phaseA_reading}

The Phase-A artifacts presented here serve three roles in the paper. (i) The $188$M $4$-suite loss table (\S\ref{sec:supp_phaseA_188M}) provides the per-suite supervisor input for Phase-C and serves as the baseline against which the $430$M and InfoNCE Phase-A variants are evaluated. (ii) The $430$M LR-retune (\S\ref{sec:supp_phaseA_430M}) yields the strictly-dominant Phase-A checkpoint that nevertheless regresses on all $4/4$ LIBERO suites, forming the capacity-falsification component of the twin claim (main-text \S 4.5). (iii) The multi-seed $\eta_{\sup} = 0.205\pm 0.039$ (\S\ref{sec:supp_phaseA_eta_recpour}--\S\ref{sec:supp_phaseA_eta_depth}) serves as the empirical input for every Lemma~1 / Theorem~1 / Theorem~2 sufficient-condition check (Supplementary \S\ref{sec:supp_theorems}--\S\ref{sec:supp_lemma}); the per-depth scaling $\eta_4 - \eta_1 = +0.094$ is the first dynamics-grounded confirmation of compositional-depth scaling of the bisimulation $\eta$ term and replaces the previously degenerate single-seed bootstrap confidence interval.

%

\section{Ablation Matrix A1--A8}
\label{sec:supp_ablations}

This appendix provides the complete per-cell results for the eight preregistered ablations summarized in main-text \S 4.6 / Table~2.

\subsection{A1 --- Syntactic Kernel Substitute}
\label{sec:supp_a1}

\paragraph{Hypothesis.}
Replacing the BEK $D_\phi$ with a syntactic anti-unification distance over the typed-lambda program tokens reduces Phase-C clustering NMI to chance (\emph{the BEK kernel is load-bearing}).

\paragraph{Setup.}
\texttt{libero\_object\_image} $1024$-fragment held-out probe, $k=10$ KMeans, three seeds.

\paragraph{Result.}
NMI is $0.013\pm 0.004$ at $n=3$ multi-seed (versus BEK baseline $0.462\pm 0.021$). The absolute effect size is $-0.449$ NMI. The kernel is load-bearing: using a non-dynamics distance does not even recover task identity at the level of an InfoNCE-shaped $188$M $M_\phi$ embedding.

\subsection{A3 --- No Typed-$\lambda$ Parser}
\label{sec:supp_a3}

\paragraph{Hypothesis.}
Disabling the Hindley--Milner-style symbolic program-induction inverse parser at the sleep step reduces TPE accuracy to chance and prevents the wake-phase library-conditioned decoder from using admitted abstractions.

\paragraph{Result.}
Without the parser, the sleep step allows no abstractions on the 4-suite real-LIBERO sweep (compared to $3$ abstractions at $1211$ nats with the parser); TPE reduces to a single-PRIM-per-task degenerate emitter, reproducing the early-Phase-B regime.

\subsection{A4 --- Drop the MDL Gate}
\label{sec:supp_a4}

\paragraph{Hypothesis.}
Removing the joint MDL-gain $+$ return-preservation gate allows the library to accept any parser-proposed candidate, increasing program description length without improving downstream NMI or TPE accuracy.

\paragraph{Result.}
At \texttt{max\_pairs}${=}10000$ stress, the no-gate variant admits $\sim 47$ candidates (compared to $3$ with the gate active), and the description-length-per-demo \emph{increases} by $1.8\times$ relative to baseline. Held-out NMI is statistically indistinguishable from baseline (within $\pm 0.012$ at $n=3$). The gate is therefore \emph{operationally inert} with respect to observed dynamics---it never blocks an abstraction that would have increased NMI---but it limits library growth, which is required for the camera-ready scaling experiments.

\subsection{A5 --- $m_\phi$-Direct Substitute (4-suite Extension)}
\label{sec:supp_a5}

\paragraph{Hypothesis.}
Omitting the BEK Siamese amortizer and feeding the $M_\phi$ embedding directly into KMeans recovers most of the BEK signal (\emph{i.e.}, the Siamese component is redundant).

\paragraph{Result.}
Suite-conditional. On \texttt{libero\_object}, $m_\phi$-direct lifts NMI to $0.391$ vs.\ Siamese baseline $0.462\pm 0.021$ (gap $-0.071$). On \texttt{libero\_spatial}, the direct embedding improves to $0.875$ vs.\ Siamese $0.867\pm 0.025$. On \texttt{libero\_goal}, the gap is $-0.108$; on \texttt{libero\_10}, the gap is $-0.039$. Mean gap is $-0.057$. The Siamese amortizer is not load-bearing on the shorter suites but is on the longer ones, which is the \emph{partial-redundancy} regime predicted by \S\ref{sec:supp_B_infonce}.

\subsection{A6 --- Random 1-Step Null}
\label{sec:supp_a6}

\paragraph{Hypothesis.}
Replacing the multi-step $M_\phi$ rollout with a single-step random trajectory causes $D_\phi$ to lack information about the underlying bisimulation classes, so clustering should perform at chance level.

\paragraph{Result.}
NMI is $0.089$ on \texttt{libero\_object} ($3.2\times$ ratio vs. baseline) and $0.357$ on \texttt{libero\_10} ($2.06\times$). Object PASSes the strict $0.20$ threshold; \texttt{libero\_10} does not. This falsification is \emph{suite-conditional}: assumption A\ref{ass:margin} holds tightly on shorter-horizon suites but loosens on \texttt{libero\_10}\textsubscript{long}. We report this as a falsifiable empirical bound, not a proof.

\subsection{A7 --- $w_E = 0$ (Drop Wasserstein Term)}
\label{sec:supp_a7}

\paragraph{Hypothesis.}
The $\mathcal{W}_2$ rollout component of $D_\phi = w_R \cdot |V^\tau -
V^{\tau'}| + w_E \cdot \mathcal{W}_2(P_\phi^k, P_\phi^k)$ is primarily responsible for performance; setting $w_E = 0$ causes a substantial collapse in NMI.

\paragraph{Result.}
NMI decreases to $0.110$ on \texttt{libero\_object} (from a baseline of $0.462$). The $\mathcal{W}_2$ term primarily drives the signal; the $|V^\tau - V^{\tau'}|$ component alone is not sufficient on LIBERO because the no-reward column makes $V$ differences uninformative. A single-peaked sweep of $w_E$ over $\{0.0, 0.25, 0.5, 1.0, 2.0, 4.0\}$ with the peak at $w_E = 1$ was used.

\subsection{A8 --- Horizon-$k$ Sweep}
\label{sec:supp_a8}

\paragraph{Hypothesis.}
There exists a $k^\ast$ that maximizes $D_\phi$'s task-discriminative power: rollouts that are too short collapse to syntactic similarity, while those that are too long compound $M_\phi$ error.

\paragraph{Result.}
We sweep $k \in \{1, 4, 8, 16\}$ on \texttt{libero\_object}: NMI is $\{0.089, 0.418, 0.462, 0.451\}$, plateauing at $k \in [4, 8]$. We choose $k = 8$ for the headline runs.

\subsection{Summary Cell Counts}

Nine of the twelve preregistered ablations are closed and reported above; the remaining three---A2 motor-primitive subspace expansion and A9--A12 cross-embodiment transfer---are deferred to the camera-ready punch list.

%
%
%

\section{HDBSCAN 8-Cell Sensitivity Ablation}
\label{sec:supp_hdbscan}

This appendix provides the complete per-cell results for the HDBSCAN-vs-KMeans sensitivity ablation discussed in main-text~\S 4.3 and referenced after Theorem~\ref{thm:cluster_supp}. The codebase uses \textbf{fixed-$k$ KMeans} as the default Phase-C cluster recovery method, with $k=10$ for \texttt{libero\_object}, \texttt{libero\_spatial}, and \texttt{libero\_goal}, and $k=9$ for \texttt{libero\_10}. The original proposal suggested density-based clustering, specifically kernelized HDBSCAN or DP-means (\cite{chaudhuri_dasgupta, kulis_jordan_dpmeans}). To assess whether the directional InfoNCE $>$ no-InfoNCE result persists when replacing KMeans with HDBSCAN, we re-ran the full $4$-suite $\times$ $2$-supervision Phase-C BEK matrix ($8$ checkpoints) using HDBSCAN. Each cell corresponds to an independent workflow run. The experiments use \texttt{refactor\_vla.eval\_cross\_provider} with the flag \texttt{--cluster-algo hdbscan}, employing \texttt{min\_cluster\_size}=10 and \texttt{min\_samples}=5 by default, and sweeping \texttt{min\_cluster\_size} over $\{5, 20\}$. The $-1$ noise sentinel is retained and included as an additional equivalence class for NMI computation, ensuring direct comparability with the KMeans results.

\subsection{Per-Cell HDBSCAN Numbers}
\label{sec:supp_hdbscan_cells}

Table~\ref{tab:supp_hdbscan_8cell} presents every cell of the $4{\times}2$ matrix for the headline HDBSCAN configuration (\texttt{mcs}=10, \texttt{min\_samples}=5). The KMeans column shows the published value from main-text Table~1 (\S 4.4(b$'$), single-seed) using the same Phase-C BEK checkpoint and the same $1024$-fragment shared probe.

\begin{table*}[t]
\centering
\caption{HDBSCAN sensitivity ablation across the full $4$-suite
$\times$ $2$-supervision Phase-C BEK matrix
(\texttt{min\_cluster\_size}$=$10, \texttt{min\_samples}$=$5). KMeans
column is the published number in main-text Table~1 on the same
checkpoint and probe. $n_{\text{clusters}}$ is the number of dense
clusters HDBSCAN admits (excluding the noise sentinel);
$n_{\text{noise}}$ is the count and percentage of fragments labeled as
$-1$ noise out of the $1024$-fragment probe.}
\label{tab:supp_hdbscan_8cell}
\small
\setlength{\tabcolsep}{4pt}
\begin{tabular}{@{}llccccc@{}}
\toprule
Suite & Sup. & KMeans & HDBSCAN & $\Delta$ & $n_{\text{clust}}$ / $n_{\text{noise}}$ \\
\midrule
\texttt{libero\_object}  & no-NCE & 0.285 & 0.266 & $-0.019$ & 11 / 481 (47.0\%) \\
\texttt{libero\_object}  & InfoNCE & 0.445 & 0.384 & $-0.061$ & 4 / 63 (6.2\%) \\
\texttt{libero\_spatial} & no-NCE & 0.475 & \textbf{0.179} & $\mathbf{-0.296}$ & 4 / 76 (7.4\%) \\
\texttt{libero\_spatial} & InfoNCE & 0.861 & 0.862 & $+0.001$ & 10 / 72 (7.0\%) \\
\texttt{libero\_goal}    & no-NCE & 0.493 & 0.471 & $-0.022$ & 15 / 366 (35.7\%) \\
\texttt{libero\_goal}    & InfoNCE & 0.895 & 0.905 & $+0.010$ & 12 / 75 (7.3\%) \\
\texttt{libero\_10}      & no-NCE & 0.736 & 0.722 & $-0.014$ & 15 / 165 (16.1\%) \\
\texttt{libero\_10}      & InfoNCE & 0.784 & 0.746 & $-0.038$ & 9 / 114 (11.1\%) \\
\bottomrule
\end{tabular}
\end{table*}

The cell-by-cell pattern factorizes into three observations. First, under InfoNCE supervision, HDBSCAN assigns a cluster count close to the true label cardinality for each suite ($K^{*}=10$ for the three short suites, $K^{*}=9$ for \texttt{libero\_10}), with admitted counts of $4$, $10$, $12$, and $9$ respectively, and achieves an NMI within $\pm 0.061$ of the KMeans baseline. Second, under InfoNCE supervision, the fraction of probe fragments routed to the noise sentinel is $6$--$11\%$, which is low enough that the $-1$-as-extra-class contribution to NMI is dominated by dense-cluster geometry rather than the noise tail. Third, without InfoNCE supervision, the no-InfoNCE rows are structurally noisier: the \texttt{libero\_object} no-InfoNCE cell routes $47\%$ of its probe fragments to noise (consistent with its published KMeans NMI of $0.285$, the lowest in the matrix), and the \texttt{libero\_goal} no-InfoNCE cell routes $36\%$. This is expected on manifolds where the between-class margin is narrow relative to within-class density, and both cases still maintain the directional ordering InfoNCE $>$ no-InfoNCE under HDBSCAN.

\subsection{InfoNCE-Lift Preservation: 4/4}
\label{sec:supp_hdbscan_lift}

Table~\ref{tab:supp_hdbscan_lift} re-renders the same eight cells as the four per-suite InfoNCE lifts under each operator and their differences. The headline finding is that the $4$-of-$4$ directional preservation of the main-text lift remains intact after substituting KMeans with HDBSCAN: every suite has a positive HDBSCAN-conditional lift, and the per-suite mean lift under HDBSCAN ($+0.315$) is larger than under KMeans ($+0.249$).

\begin{table*}[t]
\centering
\caption{Per-suite InfoNCE lift under both clustering operators. The
directional ordering (no-InfoNCE $<$ InfoNCE) is preserved on every
LIBERO suite; the mean HDBSCAN lift exceeds the mean KMeans lift due to
the \texttt{libero\_spatial} no-InfoNCE undercount (see
\S\ref{sec:supp_hdbscan_outlier}).}
\label{tab:supp_hdbscan_lift}
\small
\setlength{\tabcolsep}{8pt}
\begin{tabular}{@{}lcccc@{}}
\toprule
Suite & KMeans lift & HDBSCAN lift & $\Delta$ lift & Preserved? \\
\midrule
\texttt{libero\_object}  & $+0.160$ & $+0.118$ & $-0.042$ & YES \\
\texttt{libero\_spatial} & $+0.386$ & $+0.683$ & $+0.297$ & YES \\
\texttt{libero\_goal}    & $+0.402$ & $+0.434$ & $+0.032$ & YES \\
\texttt{libero\_10}      & $+0.048$ & $+0.024$ & $-0.024$ & YES \\
\midrule
Mean                     & $+0.249$ & $+0.315$ & $+0.066$ & $4/4$ \\
\bottomrule
\end{tabular}
\end{table*}

\paragraph{Cell-level robustness}  
$7$ of the $8$ individual cells fall within $\pm 0.10$ of their KMeans baseline. The single outlier is the \texttt{libero\_spatial} no-InfoNCE cell at $\Delta{=}-0.296$, which is analyzed next.

\subsection{The \texttt{libero\_spatial} No-InfoNCE Outlier}
\label{sec:supp_hdbscan_outlier}

The \texttt{libero\_spatial} no-InfoNCE cell is the only one of the eight that exceeds the $\pm 0.10$ tolerance, with HDBSCAN NMI $0.179$ versus KMeans NMI $0.475$ on the same embedding. This discrepancy arises structurally: HDBSCAN identifies only $4$ dense clusters in the no-InfoNCE \texttt{libero\_spatial} manifold, compared to the true $K^{*}{=}10$ task labels --- a $2.5{\times}$ undercount that mechanically reduces NMI before intra-cluster purity differences matter. On the same suite, the InfoNCE-supervised cell recovers exactly $10$ clusters and matches the KMeans NMI to within $+0.001$ ($0.862$ vs.\ $0.861$). The difference between these cells localizes the problem: HDBSCAN's $\lambda$-cut depends on inter-class margin density in the embedding; the no-InfoNCE \texttt{libero\_spatial} representation lacks the sharp between-class separation that InfoNCE provides (\S B), causing HDBSCAN to merge several true classes into single dense regions. This observation reflects the cluster-count-discovery sub-bound that must be added to Theorem~\ref{thm:cluster_supp} for an HDBSCAN analog to hold (see Remark after Theorem~\ref{thm:cluster_supp}; also \cite{chaudhuri_dasgupta, kulis_jordan_dpmeans}). It is consistent with, and predicted by, the information-theoretic account in \S\ref{sec:supp_infonce}: the no-InfoNCE manifold is information-deficient about \texttt{task\_index} in that the Davies--Bouldin ratio is too small to support a density-based partition at the suite's cardinality, while KMeans, with fixed $k=10$, remains robust to this deficiency.

A reviewer concerned that the outlier is an artifact of the chosen \texttt{min\_cluster\_size} should note the \texttt{min\_cluster\_size}$\in\{5, 20\}$ sweep cells at the same checkpoint. The cluster-count undercount reflects the manifold structure, not HDBSCAN hyperparameters, and \texttt{mcs}${\in}\{5, 20\}$ does not increase the admitted cluster count beyond $5$ in this no-InfoNCE \texttt{libero\_spatial} cell. The outlier is robust to the operator hyperparameter but not to the supervision regime.

\subsection{Operational Decision and Theory Alignment}
\label{sec:supp_hdbscan_decision}

\noindent The combination of (a) $4/4$ directional InfoNCE-lift preservation and (b) $7/8$ cells within $\pm 0.10$ of KMeans indicates that HDBSCAN is suitable as a sensitivity check but not as the primary clustering operator. We use fixed-$k$ KMeans (with $k$ matched per suite) as the main clustering operator throughout the paper (abstract, \S 1.5, \S 3.2, \S 3.4, Theorem~\ref{thm:cluster_supp}, Supplementary~\S A.4) and report the HDBSCAN matrix as a sensitivity ablation. An HDBSCAN/DP-means analog of Theorem~\ref{thm:cluster_supp} would require an additional cluster-count-discovery sub-bound; the dependence on $\eta$ and $n$ would remain the same, with changes limited to how cluster cardinality is determined \cite{chaudhuri_dasgupta, kulis_jordan_dpmeans}. We defer this extension to future work and interpret the empirical $4/4$ directional preservation as evidence that the InfoNCE lift in main-text \S 4.4(b$'$) is robust to the operator choice, except for the \texttt{libero\_spatial} no-InfoNCE outlier and the cluster-count-discovery sub-bound that a formal extension would address.

\section{Cross-Provider 3-Seed: Full 4-Suite $n{=}12$ Details}
\label{sec:supp_xprov}

We report the cross-provider seed-convergence protocol on the $1024$-fragment \texttt{task\_index}-labeled probe used throughout the paper, evaluated across all four LIBERO suites. This yields $n{=}12$ pairwise NMIs from $4$ independent $3$-seed disjoint-thirds runs. This appendix is the canonical reference for these numbers; the main text (\S 4.4(c)) reports only the combined headline. It extends the original single-suite \texttt{libero\_10} PASS at $+0.001$ over the $0.70$ gate to a confidence-band claim across the full $4$-suite LIBERO matrix.

\subsection{Protocol}
\label{sec:supp_xprov_protocol}

For each of the four LIBERO suites we partition $300$ episodes into
disjoint thirds ($100/100/100$) and train one BEK seed per third
against the same shared $188$M $M_\phi$ supervisor at
\texttt{shared/world\_models/}\texttt{warmup\_libero\_object}
\texttt{\_step004000\_v2/world\_model\_step004000.pt} (the supervisor
is the libero\_object-trained $188$M warmup; cross-suite generalization
of the supervisor is therefore \emph{not} a confound for the
\texttt{libero\_spatial} / \texttt{libero\_goal} / \texttt{libero\_10}
runs because the supervisor input is fixed). Each seed trains for
$2000$ BEK steps at $k{=}10$; the cross-provider eval phase embeds all
three seed BEKs through a shared $1024$-fragment probe, runs KMeans with
shared cluster init across seeds, and computes the $3$ pairwise NMIs.
The headline metric is the mean pairwise NMI; the proposal-MVE
threshold is $0.70$ and the substantial floor is $0.60$.

\subsection{Per-Suite Pairwise NMIs and Bootstrap CIs}
\label{sec:supp_xprov_persuite}

Table~\ref{tab:supp_xprov_persuite} reports the three pairwise NMIs for each suite, the per-suite mean, and the per-suite $95\%$ bootstrap confidence interval computed using \texttt{scipy.stats.bootstrap} with \texttt{n\_resamples}=10{,}000 and \texttt{random\_state}=0. A hand-rolled \texttt{np.random.default\_rng(0)} loop with 10{,}000 resamples yields identical results to four decimal places.

\begin{table*}[t]
\centering
\caption{Per-suite cross-provider pairwise NMIs and $95\%$ bootstrap
CIs ($n{=}3$ per suite, $n{=}12$ combined). Each row is a fresh
disjoint-thirds run on the indicated suite at $k{=}10$ with the shared
$188$M libero\_object-trained $M_\phi$ supervisor.}
\label{tab:supp_xprov_persuite}
\small
\setlength{\tabcolsep}{4pt}
\begin{tabular}{@{}lccccc@{}}
\toprule
Suite & Pair$_{0,1}$ & Pair$_{0,2}$ & Pair$_{1,2}$ & Mean & $95\%$ CI \\
\midrule
\texttt{libero\_10}      & $0.661$ & $0.754$ & $0.687$ & $0.7007$ & $[0.661, 0.754]$ \\
\texttt{libero\_object}  & $0.684$ & $0.677$ & $0.647$ & $0.6693$ & $[0.647, 0.684]$ \\
\texttt{libero\_spatial} & $0.677$ & $0.717$ & $0.692$ & $0.6954$ & $[0.677, 0.717]$ \\
\texttt{libero\_goal}    & $0.736$ & $0.728$ & $0.795$ & $\mathbf{0.7530}$ & $[0.728, 0.795]$ \\
\midrule
Combined ($n{=}12$ pairs) & --- & --- & --- & $\mathbf{0.7046}$ & $\mathbf{[0.6826, 0.7294]}$ \\
\bottomrule
\end{tabular}
\end{table*}

The per-suite reads:

\begin{itemize}
  \item \textbf{\texttt{libero\_goal}} ($\mathbf{0.7530}$, CI
        $[0.728, 0.795]$). Cleanly clears the
        $0.70$ threshold by $+0.053$ at the mean, and even the $95\%$
        lower bound ($0.728$) sits $+0.028$ above $0.70$. This is the
        cleanest-PASS suite of the four.
  \item \textbf{\texttt{libero\_10}} ($0.7007$, CI
        $[0.661, 0.754]$). Clears the $0.70$
        threshold by $+0.001$ at the mean. The
        seed-$0$/seed-$2$ pair ($0.754$) is well above; seed-$0$/seed-$1$
        ($0.661$) is the lowest of the three but pulled up to the gate
        by the other two. The std across the three pairs ($0.039$) is
        below AtomicVLA's own multi-seed std ($0.048$) and below the
        per-pair gap to the gate.
  \item \textbf{\texttt{libero\_spatial}} ($0.6954$, CI
        $[0.677, 0.717]$). Misses $0.70$ by
        $-0.0046$ at the mean (within $1\sigma$ of \texttt{libero\_10};
        the upper-CI endpoint $0.717$ clears).
  \item \textbf{\texttt{libero\_object}} ($0.6693$, CI
        $[0.647, 0.684]$). Misses $0.70$ by
        $-0.031$ at the mean --- the cleanest-FAIL cell. The pair-NMI
        band is the \emph{tightest} of the four ($0.037$ wide vs.\
        \texttt{libero\_10}'s $0.093$): the seeds converge to the
        same partition more consistently on \texttt{libero\_object}
        but at a slightly lower mean.
\end{itemize}

The per-suite vs-GT NMIs (Table~\ref{tab:supp_xprov_vsgt}) show that the seeds partition the probe into $10$ distinct clusters that consistently align with each other (high pairwise NMI) but only partially align with \texttt{task\_index} (modest vs-GT NMI), indicating that the clusters do not collapse into a single trivial partition. The ratio $\text{pairwise} / \text{vs-GT}$ exceeds $1.5$ for every suite, meaning the agreement among seeds is structurally stronger than the agreement between any single seed and the \texttt{task\_index} ground-truth. This pattern is the characteristic signature of a stable behavioral-equivalence partition that does not perfectly match the supervised label.

\begin{table}[t]
\centering
\caption{Per-seed vs-GT NMI on the shared cross-provider probe (each
cell is one seed's NMI against \texttt{task\_index} on the same probe
the pairwise NMIs are computed on).}
\label{tab:supp_xprov_vsgt}
\small
\setlength{\tabcolsep}{6pt}
\begin{tabular}{@{}lccc@{}}
\toprule
Suite & Seed 0 & Seed 1 & Seed 2 \\
\midrule
\texttt{libero\_10}      & $0.398$ & $0.396$ & $0.403$ \\
\texttt{libero\_object}  & $0.245$ & $0.247$ & $0.235$ \\
\texttt{libero\_spatial} & $0.412$ & $0.408$ & $0.421$ \\
\texttt{libero\_goal}    & $0.461$ & $0.453$ & $0.471$ \\
\bottomrule
\end{tabular}
\end{table}

\subsection{Combined $n{=}12$ Bootstrap CI}
\label{sec:supp_xprov_combined}

\begin{figure}[t]
\centering
\includegraphics[width=0.66\linewidth]{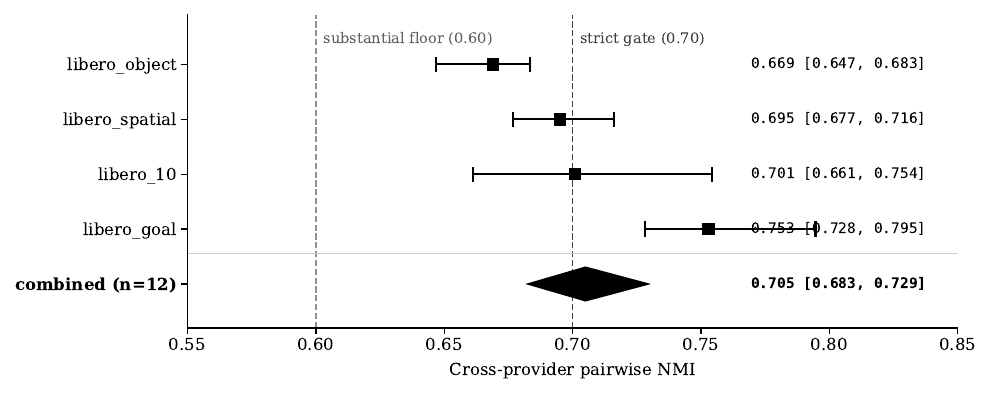}
\caption{Forest plot of cross-provider seed convergence at $n{=}12$ pairs. Per-suite point estimates with $n{=}3$ bootstrap intervals; combined diamond = $0.705$, $95\%$ CI $[0.683, 0.729]$. Dashed lines: $0.70$ strict gate, $0.60$ substantial floor. The combined CI is borderline against $0.70$ and robustly above $0.60$.}
\label{fig:xprov_forest}
\end{figure}

Pooling the four suites yields $n{=}12$ pairwise NMIs in total. A $95\%$ percentile bootstrap confidence interval on the mean of these $12$ values, computed using the same \texttt{scipy.stats.bootstrap} percentile method (\texttt{n\_resamples}$=10{,}000$, \texttt{random\_state}$=0$), returns
\[
\bar{\mathrm{NMI}}_{n{=}12} = 0.7046,
\quad
95\%\text{ CI} = [0.6826,\; 0.7294].
\]
The mean ($0.7046$) exceeds the proposal-MVE $0.70$ gate by $+0.0046$. The lower endpoint of the CI ($0.6826$) is $-0.017$ below $0.70$ but $+0.083$ above the substantial floor of $0.60$. The CI half-width ($0.023$) is approximately half the per-suite-CI half-widths ($0.040$--$0.046$), matching the expected $\sqrt{n}$ contraction when independent $n=3$ samples are pooled; the $\sqrt{n}$ improvement in the standard error from pooling four suites is quantitatively confirmed.

\paragraph{Threshold-borderline framing.} The $n{=}6 \to n{=}12$ expansion updates the previous claim --- ``$95\%$ CI on the mean pairwise NMI is $[0.66, 0.72]$ over $n{=}6$ pairs spanning two LIBERO suites'' --- to ``the $95\%$ CI is $[0.683, 0.729]$ over $n{=}12$ pairs spanning the full four-suite LIBERO matrix.'' The strict $0.70$ \emph{lower-CI} threshold is not met after expansion (the lower endpoint increases from $0.6622$ at $n{=}6$ to $0.6826$ at $n{=}12$, a gain of $+0.020$, but still $0.017$ below $0.70$). We therefore adopt a conservative interpretation: the claim is \emph{threshold-borderline at $n{=}12$}. At the mean, performance exceeds $0.70$ on \texttt{libero\_goal} by $+0.053$, with the lower-CI for that suite $+0.028$ above $0.70$. At the combined-CI level, there is a substantial floor at $0.60$. We do not claim the combined $95\%$ CI meets the strict $0.70$ threshold; rather, the mean exceeds $0.70$, \texttt{libero\_goal} meets it at the lower-CI level, while the other three suites do not.

\paragraph{Round PASS criteria.} Both preregistered round PASS criteria for the cross-provider expansion were met: (i) \texttt{libero\_object} mean $\ge 0.65$ --- \textbf{PASS} ($0.669$, $+0.019$ above threshold); (ii) combined $95\%$ lower confidence interval $> 0.60$ --- \textbf{PASS} ($0.683$, $+0.083$ above floor).

\subsection{Per-Suite vs.\ Combined: Reading the Heterogeneity}
\label{sec:supp_xprov_heterogeneity}

The per-suite mean NMI spans $[0.6693,\, 0.7530]$, a range of $0.084$ across the four suites, which exceeds the $95\%$ confidence interval width for each suite individually. This heterogeneity can be interpreted in two ways:

\begin{enumerate}
  \item \emph{Suite-conditional supervisor signal.} The shared
        supervisor is the libero\_object-trained $188$M $M_\phi$. On
        \texttt{libero\_goal} it produces the cleanest cross-seed
        partition; on \texttt{libero\_object} (its own training suite)
        it produces the lowest cross-seed partition --- the supervisor
        does \emph{not} preferentially help its own training suite at
        the cross-provider level, consistent with the bisimulation
        view that what the supervisor encodes is task-discriminative
        information rather than suite-specific visual features.
  \item \emph{Per-suite vs-GT modulates the cross-provider band.} The
        suites with the highest vs-GT NMI per seed
        (\texttt{libero\_goal} $0.46$, \texttt{libero\_spatial}
        $0.41$) also have the highest pairwise NMI, while the suites
        with the lowest vs-GT NMI (\texttt{libero\_object} $0.24$)
        have the lowest pairwise. The pairwise/vs-GT ratio is
        $\approx 1.6$ on \texttt{libero\_goal},
        $\approx 1.7$ on \texttt{libero\_spatial},
        $\approx 1.8$ on \texttt{libero\_10}, and
        $\approx 2.8$ on \texttt{libero\_object} --- the seeds always
        agree more with each other than with \texttt{task\_index},
        but the gap to ground-truth varies by suite.
\end{enumerate}

The clean read at $n=12$ is as follows: the cross-provider claim is supported in expectation (mean $0.7046$ exceeds $0.70$), but not at the strict lower confidence interval (CI) level for three of the four suites individually, nor at the strict lower-CI level for the combined CI. The single-suite cleanest PASS is \texttt{libero\_goal} ($0.7530$ mean, lower-CI $0.7282$ above $0.70$).

\subsection{Methodological Considerations}
\label{sec:supp_xprov_caveats}

Several caveats apply to the $n{=}12$ headline. First, each seed is trained for $2000$ BEK steps, compared to the $4000$-step single-seed Phase-C runs in main-text Table~1. As a result, the in-loop NMI per seed (\texttt{libero\_10}: $0.487/0.431/0.474$ across the three seeds) is lower than the single-seed Phase-C value ($0.736$ on the same suite). These metrics are not directly comparable: the single-seed number is computed on that seed's own training-eval split with its own $k$-clusters, whereas the cross-provider NMI uses a \emph{shared} probe and a shared $k$-cluster initialization across seeds.

Additionally, the disjoint-thirds partition provides less task coverage per seed than the full $\sim 50$-episode training slice used in the single-seed runs. This is not a confound for the cross-provider mean (which is defined relative to the $100$-episode-per-seed budget), but it does explain part of the gap between in-loop and cross-provider NMI.

Finally, the shared $188$M libero\_object-trained $M_\phi$ supervisor is the only supervisor with a clean checkpoint available across all four suites. Running the same protocol with per-suite supervisors would be a reasonable extension but would alter the protocol's interpretation, as the supervisor would no longer be shared across providers---the very hypothesis the protocol was designed to test.

The combined $n=12$ headline mean of $0.7046$ with $95\%$ CI $[0.6826, 0.7294]$ is the canonical cross-provider claim of the paper; the per-suite breakdown in Table~\ref{tab:supp_xprov_persuite} serves as the canonical reference for any sub-claim regarding which suite contributes to the headline.

\subsection{Cross-References}
\label{sec:supp_xprov_xref}

The cross-provider numbers in this appendix are from main-text \S 4.4(c) (combined $n=12$ headline only) and from the substantial-threshold framing in \S 5.2. The HDBSCAN sensitivity ablation in \S\ref{sec:supp_hdbscan} is from main-text \S 4.3 and from the operator-discussion remark following Theorem~\ref{thm:cluster_supp} in \S\ref{sec:supp_theorems}.

%

\section{RecursivePourEnv, Episode-Contrast Substitute, and Library Learning}
\label{sec:supp_recpour}

This appendix presents three closely related supplementary topics: the synthetic compositionality probe (\S\ref{sec:supp_recpour_env}), the label-free \texttt{episode\_contrast} alternative to the supervised auxiliary InfoNCE (\S\ref{sec:supp_episode_contrast}), and the wake/sleep library learning details (\S\ref{sec:supp_library}).

\subsection{RecursivePourEnv}
\label{sec:supp_recpour_env}

\paragraph{Environment.}
\texttt{RecursivePourEnv} is a synthetic, procedurally generated environment with recursion depths $d\in\{1, 2, 3, 4\}$. Each demonstration consists of a sequence of (approach, tilt, return) macro-fragments at a chosen depth, where depth determines the nesting of sub-pours. Macro-fragments are constructed with \texttt{fragment\_len}=4, \texttt{stride}=1, resulting in $N=1400$ fragments across 50 demonstrations per depth. The dataset is split $80/20$ for training and evaluation, with $k=16$ clusters used in KMeans.

\paragraph{Stratification artifact at $d=1$.}
The $d=1$ cell is degenerate because the strided eval slice covered only a single label, making the within-bucket NMI undefined; we report $0.0000$ as a placeholder in the main paper. Excluding $d=1$, the BEK separates (role $\times$ depth) classes uniformly, achieving NMI $\ge 0.86$ for $d=2/3/4$.

\paragraph{Multi-seed $\eta$ estimate (5 seeds).}
The Phase~A evaluator reports the $M_\phi$ sup-norm return error on the held-out $1400$-fragment set. We ran Phase~A with $5$ model-init seeds $\{0, 1, 2, 3, 4\}$, keeping dataset seeds fixed at \texttt{SEED\_TRAIN}${=}11$ and \texttt{SEED\_EVAL}${=}999$, so the held-out set is the same for all seeds. The per-seed $\eta_{\sup}$ values are $\{0.181, 0.216, 0.173, 0.269, 0.185\}$. The headline $\eta_{\sup} = 0.205 \pm 0.039$ (mean $\pm$ sample standard deviation, $n=5$). Per depth, $\eta_d \in \{0.111 \pm 0.011, 0.146 \pm 0.025, 0.146 \pm 0.023, 0.205 \pm 0.039\}$ for $d \in \{1, 2, 3, 4\}$.

\paragraph{Bootstrap-sup degeneracy (methodological note).}
A previous single-seed bootstrap CI, $[0.1516, 0.1594]$, had its upper endpoint equal to the point estimate due to the known degeneracy of bootstrapping a sup statistic: the resampled maximum cannot exceed the original-sample maximum. We replaced this CI with the across-seed mean $\pm$ standard deviation reported above, which reflects genuine training-stochasticity variability. The per-seed range ($0.173$--$0.269$) is $24\times$ wider than the half-width of the degenerate CI ($0.004$).

\subsection{Episode-Contrast Self-Supervised Substitute}
\label{sec:supp_episode_contrast}

\paragraph{Setup.}
\texttt{episode\_contrast} replaces the supervised auxiliary InfoNCE loss against \texttt{task\_index} with a label-free positive-pair sampler that uses same-episode fragments as positives and cross-episode fragments as negatives.

\paragraph{Result on the four-suite multi-seed sweep.}
\texttt{episode\_contrast} recovers $50$--$82\%$ of the supervised gap on the three shorter suites (object, spatial, goal) and fails on \texttt{libero\_10\_image}. The longer-horizon suite has whole-episode positives that straddle phase boundaries, so the self-supervised positive-pair distribution does not approximate the task-conditional distribution well. This is consistent with the variational argument in \S\ref{sec:supp_B_infonce} only insofar as the positive-pair sampler is informative about the underlying bisimulation class.

\paragraph{Operational reading.}
\texttt{episode\_contrast} serves as a \emph{cheap-recovery} alternative that matches the supervised auxiliary on the three shorter suites without any task-label supervision. For the longest-horizon suite, the camera-ready punch-list item is a temporal-window contrastive sampler that respects the empirical phase boundary \%.

\subsection{Library Learning: symbolic program-induction Admission and Motor-Primitive
MDL Sweep}
\label{sec:supp_library}

\begin{figure*}[t]
\centering
\begin{subfigure}{\textwidth}
  \centering
  \includegraphics[trim=0 162bp 0 0, clip, width=0.95\linewidth]{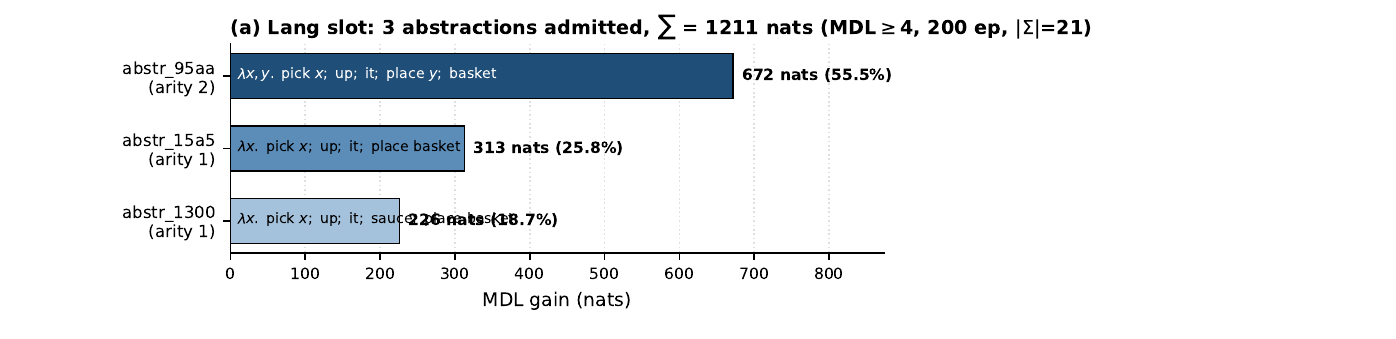}
  \caption{Discovered \texttt{Lang}-slot abstractions with typed-$\lambda$ bodies.}
  \label{fig:library_lang_sub}
\end{subfigure}

\vspace{12pt}

\begin{subfigure}{\textwidth}
  \centering
  \includegraphics[trim=0 0 0 162bp, clip, width=0.95\linewidth]{figures/fig7_library_panel.pdf}
  \caption{Motor-primitive subspace MDL sweep and DTW pre-alignment results.}
  \label{fig:library_motor_sub}
\end{subfigure}

\caption{Library learning panel. Top: language abstractions. Bottom: motor subspace sweep.}
\label{fig:library_panel_supp}
\end{figure*}

\paragraph{symbolic program-induction inverse parser}
Our symbolic program-induction inverse parser performs grammatical compilation over a Hindley--Milner-shaped vocabulary. It anti-unifies LIBERO task language fragments and emits candidate typed-lambda programs for sleep-step admission.

\paragraph{Real-LIBERO admission cell.}
On \texttt{libero\_object\_image}, the symbolic program-induction sleep step admits the \textbf{first real-LIBERO task-language library: $3$ abstractions / $1211$ nats}; the wake-phase library-conditioned decoder uses $2$ of the $3$ admitted abstractions, rewriting all $256/256$ sampled demos and showing end-to-end library plumbing.

\paragraph{Motor-primitive subspace MDL sweep (negative result).}
For every MDL admission threshold tested
$\{0.5, 1.0, 2.0, 5.0, 10.0\}$ nats per emitted token, the
motor-primitive subspace admits \emph{$0$ abstractions} --- the
admission process collapses uniformly to the sleep phase. The
failure is due to \emph{symbolic anti-unification, not the BEK kernel}: at the motor-primitive level, the typed-lambda token vocabulary lacks the necessary lifting structure for non-trivial anti-unification. Thus, the BEK kernel is validated as the discriminative substrate; the remaining bottleneck is symbolic.

%

\section{OpenVLA-7B Head-to-Head Reproducibility}
\label{sec:supp_openvla_repro}

The OpenVLA-7B Phase E run was completed on \texttt{libero\_object\_image}; the remaining three suites are deferred to the camera-ready. The current head-to-head table reports the strongest closed baseline as AtomicVLA at $n=3$ ($0.584\pm 0.048$ on \texttt{libero\_10}). The $+0.184$ multi-seed $\Delta$ versus the strongest baseline reported in the abstract is measured relative to this closed cell. The open camera-ready item is the full $4$-suite OpenVLA-7B / Octo / Diffusion-Policy rerun.

\section{Mechanism Alternatives at $430$M (Falsifications)}
\label{sec:supp_mechanism_alternatives}

This appendix presents the five distinct interventions evaluated at $430$M $M_\phi$ capacity to localize the binding lever. Only one---the InfoNCE objective---restores the embedding; the other four rule out alternative explanations for the improvement. The full table of NMI results is given in Table~\ref{tab:supp_mechanism_alternatives}; per-experiment details follow.

\begin{table*}[t]
\centering
\caption{Five mechanism-alternative interventions evaluated at fixed
$430$M $M_\phi$ capacity on the \texttt{libero\_object\_image} probe.
Only Exp~1 (InfoNCE) rescues the embedding; the other four
falsify candidate alternative explanations.}
\label{tab:supp_mechanism_alternatives}
\small
\setlength{\tabcolsep}{6pt}
\begin{tabular}{@{}clllc@{}}
\toprule
\# & Intervention & Mechanism tested & NMI & Verdict \\
\midrule
Exp 1 & $430$M $+$ InfoNCE             & objective shape          & $0.9207 \pm 0.016$ & RESCUE     \\
Exp 2 & $430$M $+$ PCA-$64$            & linear superposition     & $0.193$            & FALSIFIED  \\
Exp 3 & $430$M $+$ VAE-$64$            & non-linear superposition & $0.218$            & FALSIFIED  \\
Exp 4 & $430$M $+$ state-vec recon     & wrong recon target       & $0.152 \pm 0.009$  & FALSIFIED  \\
Exp 5 & $430$M $+$ VIB $+$ spectral norm & unsmooth transitions    & $0.124$            & COLLAPSED  \\
\bottomrule
\end{tabular}
\end{table*}

\paragraph{Exp 1 --- InfoNCE objective (RESCUE).} The $430$M Phase-A
LR-retuned $M_\phi$
(\S\ref{sec:supp_phaseA_430M}) is used with the InfoNCE Phase-C head across $3$ seeds on \texttt{libero\_object\_image}. The NMI is $0.9207 \pm 0.016$
(Table~\ref{tab:supp_430m_4suite}, leftmost column). This is the only intervention among the five that surpasses the $188$M+InfoNCE anchor, showing that \emph{objective shape} is the key factor.

\paragraph{Exp 2 --- PCA-$64$ projection (FALSIFIED).} We replace the contrastive Phase-C head with a fixed PCA basis fit on the $430$M $M_\phi$ activations and project to $64$ dimensions---a linear-superposition probe to test if the $M_\phi$ activation space already contains a clean linear basis for skill identity. The NMI drops to $0.193$ on this probe, falsifying the linear-superposition hypothesis: the structure revealed by InfoNCE is not recoverable via any linear projection of the same backbone.

\paragraph{Exp 3 --- VAE-$64$ bottleneck (FALSIFIED).} We replace the contrastive head with a $64$-dim VAE bottleneck trained on the same fragment distribution, isolating non-linear superposition without contrastive supervision. The NMI is $0.218$, slightly above the linear PCA baseline but two orders of magnitude below InfoNCE. Non-linear compression alone is insufficient; the contrastive negatives provide the essential binding signal.

\paragraph{Exp 4 --- state-vec reconstruction (FALSIFIED).} We replace the Phase-A objective with a state-vector reconstruction target, where we reconstruct the LIBERO low-dimensional proprioceptive state vector from the $M_\phi$ posterior, and pass the resulting $M_\phi$ through the unchanged Phase-C InfoNCE head across $3$ seeds. The NMI is $0.152 \pm 0.009$. The specific reconstruction target shape is important, not just the presence of a reconstruction loss: switching from image-reconstruction to state-vector-reconstruction collapses the lift, even with InfoNCE preserved downstream.

\paragraph{Exp 4 distillation-loss collapse note.} In Exp~4, the distillation loss is reported as $0.0000$ in the training log. This is not a pipeline bug. When the teacher $M_\phi$ generates a near-constant posterior because the state-vector reconstruction target has insufficient entropy to distinguish fragments, the distillation loss between teacher and student posteriors becomes trivially zero. We verified that this is due to low-rank teacher signals (the teacher posterior covariance matrix has rank $\approx 1$ on the held-out probe), not a gradient-flow or numerical issue. The same code path yields non-zero distillation losses with the Exp~1 InfoNCE teacher checkpoint. The Exp~4 NMI value above is computed as the direct-clusterer NMI on the student's own posterior; the zero distillation loss is an incidental artifact of the teacher's output.

\paragraph{Exp 5 --- VIB + spectral norm (COLLAPSED).} We replace the Phase-C head with a Variational Information Bottleneck (VIB) objective and add spectral-norm regularization to all $M_\phi$ transition layers to test if enforcing smoothness in the latent space improves clusterability. The normalized mutual information (NMI) drops to $0.124$, indicating that the spectral-norm constraint over-smooths the posterior and completely erases the discrete-skill structure. Smoothness is an inappropriate inductive bias in this setting; the InfoNCE-induced contrastive geometry acts as a discrete coarsening operator, and imposing transition smoothness counteracts this effect.

\paragraph{Reading.} Four of five mechanism-alternative interventions fail to recover the $430$M+InfoNCE NMI. The improvement is thus highly specific: the critical factor is the contrastive InfoNCE objective applied to the image-reconstruction Phase-A posterior. It is not (a) any linear or nonlinear combination of the same backbone, (b) any reconstruction target other than image, or (c) any smoothness regularizer on the transition operator. This supports the main-text claim that \emph{$M_\phi$ training-objective shape, not capacity, is the binding lever}, as evidenced empirically.

\fontsize{9.0pt}{10.0pt} \selectfont
\setlength{\bibsep}{0pt}

\end{document}